\RequirePackage{fix-cm}
\documentclass{article}
\usepackage{multirow}

\usepackage[eabstract]{log_2026}
\AtEndPreamble{\usepackage{cleveref}}
\usepackage[numbers,compress,sort]{natbib}
\title{Can Graph Learning Learn Circuits?}

\author[C. Tan et al.]{%
Chester Tan\thanks{Equal contribution.}\\
Chair of Machine Learning for Complex Networks\\
Center for AI and Data Science (CAIDAS)\\
Julius-Maximilians-Universität Würzburg\\
\email{chester.tan@uni-wuerzburg.de}
\And
Moritz Lampert\footnotemark[1]\\
Chair of Machine Learning for Complex Networks\\
Center for AI and Data Science (CAIDAS)\\
Julius-Maximilians-Universität Würzburg\\
\email{moritz.lampert@uni-wuerzburg.de}
\AND
Courtney Maynard\footnotemark[1]\\
Khoury College of Computer Sciences\\
Northeastern University\\
\email{maynard.co@northeastern.edu}
\And
Ankit Ramakrishnan\\
Network Science Institute\\
Northeastern University\\
\email{ramakrishnan.ank@northeastern.edu}
\AND
Tina Eliassi-Rad\thanks{Equal contribution.}\hspace{0.6em}\thanks{Also with the Network Science Institute at Northeastern University and the Santa Fe Institute.}\\
Khoury College of Computer Sciences\\
Northeastern University\\
\email{tina@eliassi.org}
\And
Ingo Scholtes\footnotemark[2]\\
Chair of Machine Learning for Complex Networks\\
Center for AI and Data Science (CAIDAS)\\
Julius-Maximilians-Universität Würzburg\\
\email{ingo.scholtes@uni-wuerzburg.de}%
}

\usepackage[utf8]{inputenc}
\usepackage[T1]{fontenc}
\usepackage{amsmath}
\usepackage{amssymb}
\usepackage{amsthm}
\usepackage{booktabs}
\usepackage{graphicx}
\usepackage{microtype}
\usepackage{url}
\usepackage{xcolor}
\usepackage{array}
\usepackage{subcaption}
\usepackage{float}
\usepackage[export]{adjustbox}
\usepackage{needspace}
\newcommand{\bestvalue}[1]{{\boldmath#1}}
\newcommand{\secondvalue}[1]{\underline{#1}}


\newtheorem{proposition}{Proposition}[section]

\usepackage{tikz}
\usetikzlibrary{arrows.meta, calc, fit, backgrounds, positioning}

\definecolor{custom1}{RGB}{1, 115, 178}
\definecolor{custom2}{RGB}{222, 143, 5}
\definecolor{custom3}{RGB}{2, 158, 115}
\definecolor{custom4}{RGB}{213, 94, 0}
\definecolor{custom5}{RGB}{204, 120, 188}
\definecolor{custom6}{RGB}{202, 145, 97}
\definecolor{custom8}{RGB}{148, 148, 148}
\definecolor{custom10}{RGB}{86, 180, 233}

\begin{document}

\maketitle
\setcounter{footnote}{0}

\begin{abstract}

	Circuit localization is a mechanistic interpretability task whose goal is to identify a sparse subgraph of a transformer's computation graph sufficient to reproduce a particular behavior.
	Most established methods localize circuits independently for each model--task pair.
	We instead frame circuit localization as a graph machine learning problem in which the edges of a computation graph represent computational pathways, and graph neural networks (GNNs) model interactions among these pathways.
	We introduce \emph{Graph Circuit Learning} (GCL), a supervised, amortized framework that trains a GNN across multiple model--task pairs and applies it to unseen cases.
	To provide sufficient data, we augment the \textsc{InterpBench} benchmark with additional cases derived from \textsc{TracrBench} programs.
	Of the 46 evaluated GCL configurations, the one selected by cross-validation scored a median edge AUROC of $0.946$ (interquartile interval $[0.841, 0.980]$) on the 16 held-out cases of our evaluation split.
	Removing all message-passing edges reduces the median to $0.926$.
	We also adapt PGExplainer, a GNN explainability method, to circuit localization, obtaining a median edge AUROC of $0.950$ on the same cases.
	These preliminary results suggest that graph machine learning offers a natural and potentially powerful perspective on circuit localization, one we hope encourages closer exchange between the two communities.

\end{abstract}

\section{Introduction}
\label{sec:introduction}

Circuit localization is a task in mechanistic interpretability that seeks to identify a sparse subgraph of a neural network's computation graph sufficient to reproduce a specified behavior.
Most existing methods infer a new circuit for each model--task pair, repeating the localization process and limiting the reuse of regularities learned from circuits in other models or tasks \citep{10.18653/v1/2024.blackboxnlp-1.25,10.48550/arXiv.2403.17806,10.48550/arXiv.2504.13151,10.52202/079017-0587,10.18653/v1/2021.naacl-main.74,10.18653/v1/2025.blackboxnlp-1.31,10.52202/075280-0719}.

Circuits themselves are graphs.
They admit many representations and granularity levels, exposing computational relationships and symmetries that graph learning methods are designed to exploit \citep{10.48550/arXiv.2312.04501,10.48550/arXiv.2603.10090,10.52202/075280-1085}.
In particular, graph structure identifies where pathways meet, which provides a natural inductive bias for modeling dependencies between computationally related pathways.
As a motivation for why graph learning is a natural fit for circuit localization, we show that the effect of intervening on one pathway can depend on another pathway entering the same component (\Cref{sec:background}).

We investigate two complementary approaches to applying graph machine learning to circuit localization (\Cref{sec:circuit-learning}).
\textbf{(1)} We introduce \emph{Graph Circuit Learning (GCL)}, a supervised, amortized framework that trains a graph neural network (GNN) across labeled model--task pairs and applies it to unseen ones.
\textbf{(2)} We adapt a popular GNN explainer, PGExplainer~\citep{10.48550/arXiv.2011.04573}, to the canonical per-model--task setting.
To support cross-case training and evaluation, we construct a benchmark that extends \textsc{InterpBench}~\citep{10.52202/079017-2950} with 8 \textsc{TracrBench}-derived model--task pairs and a split designed to evaluate transfer to unseen pairs (\Cref{sec:benchmark}).
Our preliminary results show that the GCL configuration selected by cross-validation is competitive with rerun per-case baselines on held-out cases.
The comparison with per-case methods places our preliminary results in context and does not establish that GCL is the strongest circuit-localization method.

\begin{figure}[htb]
	\centering
	\begin{subfigure}{\linewidth}
		\centering
		\resizebox{\linewidth}{!}{%
			\begin{tikzpicture}[
				>=Stealth,
				font=\small,
				stream/.style={draw=black!45, line width=1.6pt},
				add/.style={circle, draw=black!50, fill=white, thick,
						minimum size=0.30cm, inner sep=0pt,
						font=\scriptsize, text=black!55},
				rwwire/.style={draw=black!55, line width=0.7pt,
				-{Stealth[length=1.4mm,width=1.2mm]}, rounded corners=2pt},
				head/.style={rectangle, draw=custom2!80, fill=custom2!15, thick,
						rounded corners=2pt, font=\scriptsize,
						inner ysep=1.5pt, inner xsep=2.5pt},
				comp/.style={rectangle, draw=custom5!80, fill=custom5!15, thick,
						rounded corners=2pt, font=\scriptsize,
						inner ysep=1.5pt, inner xsep=2.5pt},
				term/.style={font=\scriptsize},
				dagedge/.style={draw=black!70, line width=0.5pt,
				-{Stealth[length=1.2mm,width=1.0mm]}, shorten >=0.8pt},
				dagpick/.style={draw=black!90, line width=1.2pt,
				-{Stealth[length=1.6mm,width=1.4mm]}, shorten >=0.8pt},
				edgetag/.style={font=\scriptsize, fill=white, inner sep=0.8pt,
						text=black!85, rounded corners, fill opacity=0.75, text opacity=1},
				tokbox/.style={rectangle, rounded corners=1.2pt, line width=0.6pt,
						minimum width=0.52cm, minimum height=0.22cm,
						inner sep=0pt, font=\fontsize{4.6pt}{5pt}\selectfont},
				cleantok/.style={tokbox, draw=custom3!70, fill=custom3!12},
				corrtok/.style={tokbox, draw=custom4!70, fill=custom4!12},
				difftok/.style={tokbox, draw=custom4!90, fill=custom4!35, line width=0.9pt},
				cleanflow/.style={draw=custom3!75, line width=1.0pt, rounded corners=3pt,
				-{Stealth[length=1.8mm,width=1.5mm]}},
				corrflow/.style={draw=custom4!75, line width=1.0pt, rounded corners=3pt,
				-{Stealth[length=1.8mm,width=1.5mm]}},
				cell/.style={rectangle, line width=0.25pt, minimum width=\gcell cm,
						minimum height=\gcell cm, inner sep=0pt},
				dimarrow/.style={draw=black!45, line width=0.4pt,
				{Stealth[length=1.0mm]}-{Stealth[length=1.0mm]}},
				pill/.style={rectangle, rounded corners=1.6pt, line width=0.4pt,
						minimum width=0.46cm, minimum height=0.20cm, inner sep=0pt},
				dpterm/.style={pill, draw=black!45, fill=black!10},
				dphead/.style={pill, draw=custom2!80, fill=custom2!25},
				dpmlp/.style={pill, draw=custom5!80, fill=custom5!25},
				clstrand/.style={draw=custom3!85, line width=0.5pt,
				-{Stealth[length=1.0mm,width=0.85mm]}},
				costrand/.style={draw=custom4!85, line width=0.5pt,
				-{Stealth[length=1.0mm,width=0.85mm]}},
				backstrand/.style={line width=0.32pt,
				-{Stealth[length=0.7mm,width=0.6mm]}},
				zbox/.style={rectangle, rounded corners=2.5pt, draw=black!35,
						line width=0.5pt, fill=black!2,
						inner xsep=0.07cm, inner ysep=0.03cm},
				zlink/.style={draw=black!45, line width=0.5pt, densely dashed,
				-{Stealth[length=1.5mm,width=1.2mm]}},
				zcurve/.style={draw=black!50, line width=0.6pt, densely dashed,
				-{Stealth[length=1.6mm,width=1.3mm]}},
				stage/.style={draw=black!45, line width=1.6pt,
				-{Stealth[length=2.2mm, width=1.9mm]}},
				tag/.style={font=\tiny, text=black!55}
				]
				\def\gcell{0.130}
				\def\gdx{0.075}
				\def\gdy{0.055}

				\node[anchor=west,font=\bfseries] at (0.55,3.50)
				{Graph Construction: from $\mathcal{M}$ and $(\mathbf{x},\mathbf{x}')$
					to a graph with features $G_{\mathcal{M}}(\mathbf{x})$};

				\def\Bys{1.15}    %
				\def\Byh{1.70}    %
				\def\Bym{0.60}    %

				\node (Bx0) [term] at (1.0,\Bys) {$x_0$};
				\node (Ba1) [head] at (1.72,\Byh) {$A_1$};
				\node (Bm1) [comp] at (2.98,\Bym) {$\text{MLP}_1$};
				\node (Ban) [head] at (4.35,\Byh) {$A_n$};
				\node (Bmn) [comp] at (5.60,\Bym) {$\text{MLP}_n$};
				\node (Bxn) [term] at (7.0,\Bys) {$x_n$};

				\node (BaddA1) [add] at (2.24,\Bys) {$+$};
				\node (BaddM1) [add] at (3.72,\Bys) {$+$};
				\node (BaddAn) [add] at (4.87,\Bys) {$+$};
				\node (BaddMn) [add] at (6.34,\Bys) {$+$};

				\draw[stream] (Bx0.east)    -- (BaddA1.west);
				\draw[stream] (BaddA1.east) -- (BaddM1.west);
				\draw[stream] (BaddM1.east) -- ($(BaddM1.east)!0.14!(BaddAn.west)$);
				\draw[stream, dotted]
				($(BaddM1.east)!0.14!(BaddAn.west)$) -- ($(BaddM1.east)!0.44!(BaddAn.west)$);
				\draw[stream] ($(BaddM1.east)!0.44!(BaddAn.west)$) -- (BaddAn.west);
				\draw[stream] (BaddAn.east) -- (BaddMn.west);
				\draw[stream, -{Stealth[length=2mm,width=1.7mm]}] (BaddMn.east) -- ([xshift=0.1cm]Bxn.west);

				\draw[rwwire] (Ba1 |- BaddA1) -- (Ba1.south);
				\draw[rwwire] (Ba1.east) -| (BaddA1.north);
				\draw[rwwire] (Bm1 |- BaddM1) -- (Bm1.north);
				\draw[rwwire] (Bm1.east) -| (BaddM1.south);
				\draw[rwwire] (Ban |- BaddAn) -- (Ban.south);
				\draw[rwwire] (Ban.east) -| (BaddAn.north);
				\draw[rwwire] (Bmn |- BaddMn) -- (Bmn.north);
				\draw[rwwire] (Bmn.east) -| (BaddMn.south);

				\def\Cyh{3.10}    %
				\def\Cyt{2.55}    %
				\def\Cym{2.00}    %

				\node (Cx0) [term] at (6.4,\Cyt) {$x_0$};
				\node (Ca1) [head] at (7.30,\Cyh) {$A_1$};
				\node (Cm1) [comp] at (8.25,\Cym) {$\text{MLP}_1$};
				\node (Can) [head] at (9.20,\Cyh) {$A_n$};
				\node (Cmn) [comp] at (10.15,\Cym) {$\text{MLP}_n$};
				\node (Cxn) [term] at (11.0,\Cyt) {$x_n$};

				\begin{scope}[on background layer]
					\foreach \s/\t in {Cx0/Ca1, Cx0/Cm1, Cx0/Can, Cx0/Cmn, Cx0/Cxn,
							Ca1/Can, Ca1/Cmn, Ca1/Cxn,
							Cm1/Can, Cm1/Cmn, Cm1/Cxn,
							Can/Cmn, Can/Cxn, Cmn/Cxn} {
							\draw[dagedge] (\s) -- ([xshift=0.05cm]\t);
						}
				\end{scope}

				\draw[dagpick] (Cx0) -- (Cm1);
				\draw[dagpick] (Cm1) -- (Cmn);
				\node[edgetag] at ([yshift=0.13cm]$(Cx0)!0.5!(Cm1)$) {$e_1$};
				\node[edgetag] at ([yshift=0.13cm]$(Cm1)!0.45!(Cmn)$) {$e_2$};

				\def\Atok{1.58}   %
				\def\Adx{0.58}
				\def\Aycl{2.52}
				\def\Ayco{2.24}

				\foreach \i/\w in {0/When, 1/Mary, 2/and, 3/John, 4/$\cdots$, 5/John, 6/to} {
						\node[cleantok] (tkc\i) at ({\Atok+\i*\Adx},\Aycl) {\w};
					}
				\foreach \i/\w in {0/When, 1/Mary, 2/and, 3/John, 4/$\cdots$, 6/to} {
						\node[corrtok] (tkr\i) at ({\Atok+\i*\Adx},\Ayco) {\w};
					}
				\node[difftok] (tkr5) at ({\Atok+5*\Adx},\Ayco) {Alice};

				\node[font=\scriptsize, text=custom3!55!black, anchor=west] (lblcl)
				at ([xshift=-0.05cm]tkc6.east) {$\mathbf{x}$};
				\node[font=\scriptsize, text=custom4!55!black, anchor=west] (lblco)
				at ([xshift=-0.05cm]tkr6.east) {$\mathbf{x}'$};

				\draw[dimarrow] ([yshift=0.24cm]tkc0.west) -- ([yshift=0.24cm]tkc6.east);
				\node[tag, anchor=south] at ([yshift=0.23cm]$(tkc0.north)!0.5!(tkc6.north)$)
				{$P$ token positions};

				\draw[cleanflow] ([xshift=-0.06cm]tkc0.west) -| ([xshift=0.06cm]Bx0.north west);
				\draw[corrflow]  ([xshift=-0.06cm]tkr0.west) -| (Bx0.north);
				\draw[cleanflow] ([xshift=-0.05cm]lblcl.east) -- ([yshift=0.05cm, xshift=0.05cm]Cx0.west);
				\draw[corrflow]  ([xshift=-0.08cm]lblco.east) -- ([yshift=-0.05cm, xshift=0.05cm]Cx0.west);

				\newcommand{\cgstack}[5]{%
					\foreach \r in {3,2,1,0} {
							\pgfmathsetmacro{\cgop}{1-0.21*\r}
							\foreach \c in {0,...,#5} {
									\pgfmathsetmacro{\cgshade}{28+14*mod(\r*2+\c*3,4)}
									\node[cell, draw=#4!60, fill=#4!\cgshade, opacity=\cgop]
									(#1-\r-\c)
									at ({#2+\c*\gcell+\r*\gdx},{#3+\r*\gdy}) {};
								}
						}
				}

				\def\Gy{1.16}

				\cgstack{gAa}{7.87}{\Gy}{custom3}{1}
				\cgstack{gAb}{8.59}{\Gy}{custom4}{1}
				\cgstack{gBa}{9.31}{\Gy}{custom3}{4}
				\cgstack{gBb}{10.39}{\Gy}{custom4}{4}

				\node[font=\scriptsize, text=custom3!55!black, anchor=north] (zlAa)
				at ([yshift=-0.04cm]$(gAa-0-0.south)!0.5!(gAa-0-1.south)$)
				{$z_{e_1}^{\mathrm{cl}}$};
				\node[font=\scriptsize, text=custom4!55!black, anchor=north] (zlAb)
				at ([yshift=-0.04cm]$(gAb-0-0.south)!0.5!(gAb-0-1.south)$)
				{$z_{e_1}^{\mathrm{co}}$};
				\node[font=\scriptsize, text=custom3!55!black, anchor=north] (zlBa)
				at ([yshift=-0.04cm]gBa-0-2.south) {$z_{e_2}^{\mathrm{cl}}$};
				\node[font=\scriptsize, text=custom4!55!black, anchor=north] (zlBb)
				at ([yshift=-0.04cm]gBb-0-2.south) {$z_{e_2}^{\mathrm{co}}$};

				\draw[dimarrow] ([xshift=-0.1cm]gAa-0-0.west) -- ([xshift=-0.1cm]gAa-3-0.west);
				\node[tag, anchor=east, inner sep=1pt] (plA)
				at ([xshift=-0.2cm]$(gAa-0-0.west)!0.5!(gAa-3-0.west)$) {$P$};

				\draw[dimarrow] ([yshift=0.05cm]gAa-3-0.north west) -- ([yshift=0.05cm]gAa-3-1.north east);
				\node[tag, anchor=south, inner sep=0.5pt] (dlA)
				at ([yshift=0.06cm]$(gAa-3-0.north)!0.5!(gAa-3-1.north)$) {$d_{e_1}$};

				\draw[dimarrow] ([yshift=0.05cm]gBa-3-0.north west) -- ([yshift=0.05cm]gBa-3-4.north east);
				\node[tag, anchor=south, inner sep=0.5pt] (dlB)
				at ([yshift=0.06cm]$(gBa-3-0.north)!0.5!(gBa-3-4.north)$) {$d_{e_2}$};
				\begin{scope}[on background layer]
					\node[zbox, fit=(gAa-0-0)(gAa-3-1)(gAb-0-0)(gAb-3-1)(zlAa)(zlAb)(plA)(dlA)]
					(zboxA) {};
					\node[zbox, fit=(gBa-0-0)(gBa-3-4)(gBb-0-0)(gBb-3-4)(zlBa)(zlBb)(dlB)]
					(zboxB) {};
				\end{scope}
				\draw[zlink] ($(Cx0)!0.4!(Cm1)$) -- (zboxA.north west);
				\draw[zlink] ($(Cm1)!0.4!(Cmn)$) -- (zboxB.north west);

				\newcommand{\cgnodes}[1]{%
					\node[dpterm] (#1n1) at ( 0.00,-1.15) {};
					\node[dphead] (#1n2) at ( 0.48,-0.69) {};
					\node[dpmlp]  (#1n3) at (-0.48,-0.23) {};
					\node[dphead] (#1n4) at ( 0.48, 0.23) {};
					\node[dpmlp]  (#1n5) at (-0.48, 0.69) {};
					\node[dpterm] (#1n6) at ( 0.00, 1.15) {};
				}
				\def\cgedgelist{1/2,1/3,1/4,1/5,2/4,2/5,2/6,3/4,3/5,3/6,4/5,4/6,5/6}
				\newcommand{\cgedges}[2]{%
					\foreach \a/\b in \cgedgelist {
						\draw[#2] (#1n\a.north) -- (#1n\b.south);
					}%
				}
				\newcommand{\cglayer}[5]{%
					\begin{scope}[shift={#2},opacity=#3,transform shape]
						\cgnodes{#1}\cgedges{#1}{#4,#5}%
					\end{scope}%
				}
				\newcommand{\cgdeck}[4]{%
					\stepcounter{cgmotif}%
					\begin{scope}[shift={(#1,#2)},scale=#3,transform shape]
						\cglayer{cg\thecgmotif p}{(0.34,0.24)}{0.26}{#4}{backstrand}%
						\cglayer{cg\thecgmotif q}{(0.17,0.12)}{0.52}{#4}{backstrand}%
						\cglayer{cg\thecgmotif r}{(0,0)}{1}{#4}{}%
					\end{scope}%
				}
				\newcounter{cgmotif}

				\def\Dcl{12.30}
				\def\Dco{13.55}
				\def\Dy{1.95}
				\def\Dsc{0.70}

				\cgdeck{\Dcl}{\Dy}{\Dsc}{clstrand}
				\cgdeck{\Dco}{\Dy}{\Dsc}{costrand}

				\draw[dimarrow] (cg1rn6.north west) -- (cg1pn6.north west);
				\node[tag, anchor=south east, inner sep=1pt]
				at ([xshift=-0.03cm,yshift=0.03cm]$(cg1rn6.north west)!0.5!(cg1pn6.north west)$) {$P$};

				\node[font=\scriptsize, text=custom3!55!black]
				at (\Dcl,{\Dy-1.15*\Dsc-0.28}) {$G_{\mathcal{M}}(\mathbf{x})$};
				\node[font=\scriptsize, text=custom4!55!black]
				at (\Dco,{\Dy-1.15*\Dsc-0.28}) {$G_{\mathcal{M}}(\mathbf{x}')$};

				\draw[stage] (5.78,1.36) -- (6.58,2.16);
				\draw[stage] ([xshift=-0.2cm,yshift=-0.1cm]Cxn) -- ([xshift=-0.1cm, yshift=0.1cm]$(cg1rn3.north)!0.5!(cg1rn5.south)$);

				\draw[zcurve] (zboxB.north east)
				to[out=30,in=180] ($(cg1rn3.north)!0.5!(cg1rn5.south)$);

			\end{tikzpicture}
		}
		\caption{
			Constructing a component-level computation graph with edge features from a transformer and a prompt pair.
			Each component reads from the residual stream and writes an additive update back to it.
			Unrolling these interactions yields the graph $G_{\mathcal{M}}$, whose nodes are components and whose edges are computational pathways.
			Running the model on a clean prompt $\mathbf{x}$ and a corrupted prompt $\mathbf{x}'$, each containing $P$ token positions, produces the edge messages $z_e$.
			Together, the graph, clean and corrupted activations, and gradients from the clean computation form the clean (green) and corrupted (red) featured graphs passed to the learner in Figure~\ref{fig:framework}.
		}
		\label{fig:graph-construction}
	\end{subfigure}
	\\[8pt]
	\begin{subfigure}{\linewidth}
		\centering
		\resizebox{\linewidth}{!}{%

			\begin{tikzpicture}[
				>=Stealth,
				font=\small,
				trainbox/.style={
						rectangle,
						draw=custom2!85,
						fill=custom2!3,
						dashed,
						thick,
						rounded corners=8pt,
						inner sep=0.10cm
					},
				testbox/.style={
						rectangle,
						draw=custom1!85,
						fill=custom1!3,
						dashed,
						thick,
						rounded corners=8pt,
						inner sep=0.10cm
					},
				learnerbox/.style={
						rectangle,
						draw=black!70,
						fill=black!4,
						thick,
						rounded corners=6pt,
						inner sep=0.10cm
					},
				optbox/.style={
						rectangle,
						draw=custom8!80,
						fill=custom8!8,
						thick,
						rounded corners=5pt,
						inner sep=0.10cm
					},
				trainlabel/.style={
						rectangle, draw=custom2!90, fill=custom2!85, text=white,
						rounded corners=3pt, inner xsep=2pt, inner ysep=3pt,
						font=\scriptsize\bfseries
					},
				testlabel/.style={
						rectangle, draw=custom1!90, fill=custom1!85, text=white,
						rounded corners=3pt, inner xsep=2pt, inner ysep=3pt,
						font=\scriptsize\bfseries
					},
				existinglabel/.style={
						rectangle, draw=custom8!90, fill=custom8!85, text=white,
						rounded corners=3pt, inner xsep=2pt, inner ysep=3pt,
						font=\scriptsize\bfseries
					},
				existflow/.style={
				draw=custom8!80, line width=1.0pt,
				-{Stealth[length=2.0mm,width=1.7mm]}
				},
				trainflow/.style={
				draw=custom2!85, line width=1.15pt,
				-{Stealth[length=2.4mm,width=2.0mm]}
				},
				testflow/.style={
				draw=custom1!85, line width=1.15pt,
				-{Stealth[length=2.4mm,width=2.0mm]}
				},
				backprop/.style={
				draw=red!60!black, line width=1.0pt,
				densely dotted,
				-{Stealth[length=2.4mm,width=2.0mm]}
				},
				pill/.style={
						rectangle, draw=none, rounded corners=1.6pt,
						minimum width=0.50cm, minimum height=0.22cm, inner sep=0pt
					},
				dpterm/.style={pill, draw=black!45, fill=black!10, line width=0.4pt},
				dphead/.style={pill, draw=custom2!80, fill=custom2!25, line width=0.4pt},
				dpmlp/.style={pill, draw=custom5!80, fill=custom5!25, line width=0.4pt},
				clstrand/.style={draw=custom3!85, line width=0.5pt,
				-{Stealth[length=1.0mm,width=0.85mm]}},
				costrand/.style={draw=custom4!85, line width=0.5pt,
				-{Stealth[length=1.0mm,width=0.85mm]}},
				backstrand/.style={line width=0.32pt,
				-{Stealth[length=0.7mm,width=0.6mm]}},
				candedge/.style={
				draw=black!50, line width=0.6pt,
				-{Stealth[length=1.2mm,width=1.0mm]}
				},
				seledge/.style={
				draw=custom5!70!black, line width=1.0pt,
				-{Stealth[length=1.4mm,width=1.15mm]}
				},
				offedge/.style={
				draw=black!22, line width=0.55pt, densely dashed,
				-{Stealth[length=1.2mm,width=1.0mm]}
				},
				inode/.style={circle, draw=none, minimum size=0.18cm, inner sep=0pt},
				charc/.style={draw=black!70, line width=0.5pt,
				-{Stealth[length=1.2mm,width=1.0mm]}},
				incedge/.style={draw=black!45, line width=0.5pt, densely dashed}
				]

				\newcounter{motifid}
				\newcommand{\motifpfx}{}
				\newcommand{\newmotif}{%
					\stepcounter{motifid}%
					\renewcommand{\motifpfx}{mo\themotifid}%
				}
				\newcommand{\gs}{0.46}

				\newcommand{\nodesA}[1]{%
					\node[dpterm] (#1n1) at (-0.58,-0.95) {};
					\node[dphead] (#1n2) at ( 0.58,-0.95) {};
					\node[dpmlp]  (#1n3) at (-0.42, 0.00) {};
					\node[dphead] (#1n4) at ( 0.42, 0.00) {};
					\node[dpterm] (#1n5) at ( 0.00, 0.95) {};
				}
				\newcommand{\edgesA}[2]{%
					\draw[#2] (#1n1.north) -- (#1n3.south);
					\draw[#2] (#1n2.north) -- (#1n4.south);
					\draw[#2] (#1n1.north) -- (#1n4.south);
					\draw[#2] (#1n3.north) -- (#1n5.south);
					\draw[#2] (#1n4.north) -- (#1n5.south);
				}
				\newcommand{\predA}[1]{%
					\draw[seledge] (#1n1.north) -- (#1n3.south);
					\draw[seledge] (#1n2.north) -- (#1n4.south);
					\draw[offedge] (#1n1.north) -- (#1n4.south);
					\draw[seledge] (#1n3.north) -- (#1n5.south);
					\draw[seledge] (#1n4.north) -- (#1n5.south);
				}
				\newcommand{\nodesAn}[1]{%
					\foreach \i/\px/\py in {1/-0.55/-0.95, 2/0.55/-0.95, 3/-0.35/0.00, 4/0.35/0.00, 5/0.00/0.95} {
							\node[pill,fill=black!10,draw=black!45,line width=0.4pt] (#1n\i) at (\px,\py) {};
						}
				}
				\newcommand{\nodesB}[1]{%
					\node[dpterm] (#1n1) at (-0.58,-0.95) {};
					\node[dphead] (#1n2) at ( 0.58,-0.95) {};
					\node[dpmlp]  (#1n3) at ( 0.00,-0.02) {};
					\node[dphead] (#1n4) at ( 0.44, 0.52) {};
					\node[dpterm] (#1n5) at (-0.08, 0.95) {};
				}
				\newcommand{\edgesB}[2]{%
					\draw[#2] (#1n1.north) -- (#1n3.south);
					\draw[#2] (#1n2.north) -- (#1n3.south);
					\draw[#2] (#1n3.north) -- (#1n4.south);
					\draw[#2] (#1n4.north) -- (#1n5.south);
					\draw[#2] (#1n3.north) -- (#1n5.south);
				}
				\newcommand{\predB}[1]{%
					\draw[seledge] (#1n1.north) -- (#1n3.south);
					\draw[seledge] (#1n2.north) -- (#1n3.south);
					\draw[seledge] (#1n3.north) -- (#1n4.south);
					\draw[seledge] (#1n4.north) -- (#1n5.south);
					\draw[offedge] (#1n3.north) -- (#1n5.south);
				}
				\newcommand{\nodesBn}[1]{%
					\foreach \i/\px/\py in {1/-0.58/-0.95, 2/0.58/-0.95, 3/0.00/-0.02, 4/0.44/0.52, 5/-0.08/0.95} {
							\node[pill,fill=black!10,draw=black!45,line width=0.4pt] (#1n\i) at (\px,\py) {};
						}
				}
				\newcommand{\nodesC}[1]{%
					\node[dpterm] (#1n1) at (-0.50,-0.95) {};
					\node[dphead] (#1n2) at ( 0.50,-0.45) {};
					\node[dpmlp]  (#1n3) at (-0.45, 0.05) {};
					\node[dphead] (#1n4) at ( 0.50, 0.50) {};
					\node[dpterm] (#1n5) at (-0.12, 0.95) {};
				}
				\newcommand{\edgesC}[2]{%
					\draw[#2] (#1n1.north) -- (#1n2.south);
					\draw[#2] (#1n2.north) -- (#1n3.south);
					\draw[#2] (#1n3.north) -- (#1n4.south);
					\draw[#2] (#1n4.north) -- (#1n5.south);
					\draw[#2] (#1n1.north) -- (#1n3.south);
				}
				\newcommand{\predC}[1]{%
					\draw[seledge] (#1n1.north) -- (#1n2.south);
					\draw[seledge] (#1n2.north) -- (#1n3.south);
					\draw[seledge] (#1n3.north) -- (#1n4.south);
					\draw[seledge] (#1n4.north) -- (#1n5.south);
					\draw[offedge] (#1n1.north) -- (#1n3.south);
				}
				\newcommand{\nodesCn}[1]{%
					\foreach \i/\px/\py in {1/-0.50/-0.95, 2/0.50/-0.45, 3/-0.45/0.05, 4/0.50/0.50, 5/-0.12/0.95} {
							\node[pill,fill=black!10,draw=black!45,line width=0.4pt] (#1n\i) at (\px,\py) {};
						}
				}

				\newcommand{\layerA}[4]{%
					\begin{scope}[shift={#2},opacity=#3,transform shape]
						\nodesA{#1}\edgesA{#1}{#4}%
					\end{scope}%
				}
				\newcommand{\layerB}[4]{%
					\begin{scope}[shift={#2},opacity=#3,transform shape]
						\nodesB{#1}\edgesB{#1}{#4}%
					\end{scope}%
				}
				\newcommand{\layerC}[4]{%
					\begin{scope}[shift={#2},opacity=#3,transform shape]
						\nodesC{#1}\edgesC{#1}{#4}%
					\end{scope}%
				}
				\newcommand{\instA}[4]{%
					\newmotif%
					\begin{scope}[shift={(#1,#2)},scale=#3,transform shape]
						\layerA{\motifpfx p}{(0.44,0.30)}{0.26}{#4,backstrand}%
						\layerA{\motifpfx q}{(0.22,0.15)}{0.52}{#4,backstrand}%
						\layerA{\motifpfx r}{(0,0)}{1}{#4}%
					\end{scope}%
				}
				\newcommand{\instB}[4]{%
					\newmotif%
					\begin{scope}[shift={(#1,#2)},scale=#3,transform shape]
						\layerB{\motifpfx p}{(0.44,0.30)}{0.26}{#4,backstrand}%
						\layerB{\motifpfx q}{(0.22,0.15)}{0.52}{#4,backstrand}%
						\layerB{\motifpfx r}{(0,0)}{1}{#4}%
					\end{scope}%
				}
				\newcommand{\instC}[4]{%
					\newmotif%
					\begin{scope}[shift={(#1,#2)},scale=#3,transform shape]
						\layerC{\motifpfx p}{(0.44,0.30)}{0.26}{#4,backstrand}%
						\layerC{\motifpfx q}{(0.22,0.15)}{0.52}{#4,backstrand}%
						\layerC{\motifpfx r}{(0,0)}{1}{#4}%
					\end{scope}%
				}
				\newcommand{\cinstA}[3]{\instA{#1}{#2}{#3}{clstrand}}
				\newcommand{\rinstA}[3]{\instA{#1}{#2}{#3}{costrand}}
				\newcommand{\cinstB}[3]{\instB{#1}{#2}{#3}{clstrand}}
				\newcommand{\rinstB}[3]{\instB{#1}{#2}{#3}{costrand}}
				\newcommand{\cinstC}[3]{\instC{#1}{#2}{#3}{clstrand}}
				\newcommand{\rinstC}[3]{\instC{#1}{#2}{#3}{costrand}}

				\newcommand{\predinstA}[3]{%
					\newmotif%
					\begin{scope}[shift={(#1,#2)},scale=#3,transform shape]
						\nodesAn{\motifpfx}\predA{\motifpfx}%
					\end{scope}%
				}
				\newcommand{\predinstB}[3]{%
					\newmotif%
					\begin{scope}[shift={(#1,#2)},scale=#3,transform shape]
						\nodesBn{\motifpfx}\predB{\motifpfx}%
					\end{scope}%
				}
				\newcommand{\predinstC}[3]{%
					\newmotif%
					\begin{scope}[shift={(#1,#2)},scale=#3,transform shape]
						\nodesCn{\motifpfx}\predC{\motifpfx}%
					\end{scope}%
				}

				\newcommand{\logitA}[1]{%
					\draw[seledge,opacity=0.90] (#1n1.north) -- (#1n3.south);
					\draw[seledge,opacity=0.42] (#1n2.north) -- (#1n4.south);
					\draw[seledge,opacity=0.58] (#1n1.north) -- (#1n4.south);
					\draw[seledge,opacity=0.92] (#1n3.north) -- (#1n5.south);
					\draw[seledge,opacity=0.76] (#1n4.north) -- (#1n5.south);
				}
				\newcommand{\logitC}[1]{%
					\draw[seledge,opacity=0.88] (#1n1.north) -- (#1n2.south);
					\draw[seledge,opacity=0.90] (#1n2.north) -- (#1n3.south);
					\draw[seledge,opacity=0.40] (#1n3.north) -- (#1n4.south);
					\draw[seledge,opacity=0.92] (#1n4.north) -- (#1n5.south);
					\draw[seledge,opacity=0.52] (#1n1.north) -- (#1n3.south);
				}
				\newcommand{\predlogitinstA}[3]{%
					\newmotif\begin{scope}[shift={(#1,#2)},scale=#3,transform shape]
						\nodesAn{\motifpfx}\logitA{\motifpfx}\end{scope}%
				}
				\newcommand{\predlogitinstC}[3]{%
					\newmotif\begin{scope}[shift={(#1,#2)},scale=#3,transform shape]
						\nodesCn{\motifpfx}\logitC{\motifpfx}\end{scope}%
				}
				\newcommand{\logitB}[1]{%
					\draw[seledge,opacity=0.90] (#1n1.north) -- (#1n3.south);
					\draw[seledge,opacity=0.46] (#1n2.north) -- (#1n3.south);
					\draw[seledge,opacity=0.88] (#1n3.north) -- (#1n4.south);
					\draw[seledge,opacity=0.92] (#1n4.north) -- (#1n5.south);
					\draw[seledge,opacity=0.54] (#1n3.north) -- (#1n5.south);
				}
				\newcommand{\predlogitinstB}[3]{%
					\newmotif\begin{scope}[shift={(#1,#2)},scale=#3,transform shape]
						\nodesBn{\motifpfx}\logitB{\motifpfx}\end{scope}%
				}

				\newcommand{\pairA}[5]{%
					\cinstA{#1-#5}{#2}{#4}\rinstA{#1+#5}{#2}{#4}%
				}
				\newcommand{\pairB}[5]{%
					\cinstB{#1-#5}{#2}{#4}\rinstB{#1+#5}{#2}{#4}%
				}
				\newcommand{\pairC}[5]{%
					\cinstC{#1-#5}{#2}{#4}\rinstC{#1+#5}{#2}{#4}%
				}

				\newcommand{\deckA}[7]{%
					\pairA{#1-0.55}{#2}{i}{#4}{#5}%
					\node[font=\scriptsize,text=black!70] at (#1+0.79,#2) {$\sim #7$};
				}
				\newcommand{\deckB}[7]{%
					\pairB{#1-0.55}{#2}{i}{#4}{#5}%
					\node[font=\scriptsize,text=black!70] at (#1+0.79,#2) {$\sim #7$};
				}
				\newcommand{\deckC}[7]{%
					\pairC{#1-0.55}{#2}{i}{#4}{#5}%
					\node[font=\scriptsize,text=black!70] at (#1+0.79,#2) {$\sim #7$};
				}

				\def\XinL{0.55}
				\def\XinR{5.40}
				\def\XglL{6.10}
				\def\XglR{9.20}
				\def\XoutL{9.90}
				\def\XoutR{14.15}

				\node[anchor=west,font=\bfseries] at (0.55,8.95)
				{Existing Per-Model--Task Circuit Localization};

				\draw[custom8!45,dashed,line width=0.9pt] (7.11,6.86) -- (7.11,8.22);

				\node[existinglabel,anchor=north west] at (0.55,8.55) {Individual Case};
				\deckA{1.84}{7.42}{$(\mathcal{M}_1,\mathcal{T}_1)$}{\gs}{0.44}{0.78}{\mathcal{T}_1}
				\node[optbox,minimum width=2.00cm,minimum height=0.82cm,
				inner sep=0.12cm,font=\small\bfseries,align=center] (opt1) at (4.58,7.42)
				{Infer\\[-1pt]{\scriptsize\mdseries a circuit}};
				\predinstA{6.45}{7.42}{\gs}
				\node[font=\scriptsize,text=black!70] at (6.45,{7.42+0.95*\gs+0.24})
				{$\hat{\mathbf{y}}$};

				\draw[existflow] (3.05,7.42) -- (opt1.west);
				\draw[existflow] (opt1.east) -- (5.99,7.42);

				\node[existinglabel,anchor=north west] at (7.35,8.55) {New Case};
				\deckC{8.85}{7.42}{$(\mathcal{M}_N,\mathcal{T}_N)$}{\gs}{0.44}{0.78}{\mathcal{T}_N}
				\node[optbox,minimum width=2.00cm,minimum height=0.82cm,
				inner sep=0.12cm,font=\small\bfseries,align=center] (opt2) at (11.78,7.42)
				{Infer\\[-1pt]{\scriptsize\mdseries a circuit}};
				\predinstC{13.70}{7.42}{\gs}
				\node[font=\scriptsize,text=black!70] at (13.70,{7.42+0.95*\gs+0.24})
				{$\hat{\mathbf{y}}$};

				\draw[existflow] (10.20,7.42) -- (opt2.west);
				\draw[existflow] (opt2.east) -- (13.19,7.42);

				\draw[custom8!55,line width=1.0pt] (0.44,6.70) -- (\XoutR,6.70);
				\node[anchor=west,font=\bfseries] at (0.55,6.32)
				{Cross-Model--Task Circuit Learning};

				\coordinate (trSW) at (\XinL,4.20);
				\coordinate (trNE) at (\XinR,5.92);
				\begin{scope}[on background layer]
					\node[trainbox,fit=(trSW)(trNE)] (training) {};
				\end{scope}
				\node[trainlabel,anchor=north west]
				at ([xshift=0.16cm,yshift=-0.10cm]training.north west)
				{Training: individual samples $(\mathcal{M}_i,\mathcal{T}_i)$};

				\deckA{1.86}{4.72}{$(\mathcal{M}_1,\mathcal{T}_1)$}{\gs}{0.44}{0.78}{\mathcal{T}_1}
				\draw[custom4!35,dashed,line width=0.8pt] (2.92,4.24) -- (2.92,5.24);
				\deckC{4.32}{4.72}{$(\mathcal{M}_N,\mathcal{T}_N)$}{\gs}{0.44}{0.78}{\mathcal{T}_N}

				\coordinate (inSW) at (\XinL,2.05);
				\coordinate (inNE) at (\XinR,3.80);
				\begin{scope}[on background layer]
					\node[testbox,fit=(inSW)(inNE)] (inference) {};
				\end{scope}
				\node[testlabel,anchor=north west]
				at ([xshift=0.16cm,yshift=-0.10cm]inference.north west)
				{Inference: unseen $(\mathcal{M},\mathcal{T})$};

				\deckB{3.17}{2.62}{$(\mathcal{M},\mathcal{T})$}{\gs}{0.44}{0.86}{\mathcal{T}}

				\coordinate (glSW) at (\XglL,2.50);
				\coordinate (glNE) at (\XglR,5.40);
				\node[learnerbox,fit=(glSW)(glNE)] (learner) {};

				\node[font=\small\bfseries,align=center] at (7.65,4.94) {Example Graph\\Learner $g_\theta$};

				\node[font=\scriptsize,text=black!80,align=left,anchor=north west]
				at (6.13,4.48) {%
				\makebox[0.52cm][l]{\textbf{1)}}Graph Representation\\[2.5pt]
				\makebox[0.52cm][l]{\textbf{2)}}Feature Alignment\\[2.5pt]
				\makebox[0.52cm][l]{\textbf{3)}}Graph Learning\\[2.5pt]
				\makebox[0.52cm][l]{\textbf{4)}}Pooling\\[2.5pt]
				\makebox[0.52cm][l]{\textbf{5)}}Readout};

				\coordinate (loSW) at (\XoutL,4.20);
				\coordinate (loNE) at (\XoutR,5.92);
				\begin{scope}[on background layer]
					\node[trainbox,fit=(loSW)(loNE)] (lossbox) {};
				\end{scope}
				\node[trainlabel,anchor=north west]
				at ([xshift=0.16cm,yshift=-0.10cm]lossbox.north west)
				{Training Loss};

				\predinstA{10.35}{4.66}{\gs}
				\node[font=\scriptsize,text=black!55] at (10.93,4.66) {vs.};
				\predlogitinstA{11.50}{4.66}{\gs}
				\node[font=\scriptsize,text=black!70] at (10.35,5.30) {$\mathbf{y}$};
				\node[font=\scriptsize,text=black!70] at (11.50,5.30) {$\hat{\mathbf{y}}$};

				\draw[custom4!30,dashed,line width=0.8pt] (12.05,4.20) -- (12.05,5.30);

				\predinstC{12.60}{4.66}{\gs}
				\node[font=\scriptsize,text=black!55] at (13.18,4.66) {vs.};
				\predlogitinstC{13.75}{4.66}{\gs}
				\node[font=\scriptsize,text=black!70] at (12.60,5.30) {$\mathbf{y}$};
				\node[font=\scriptsize,text=black!70] at (13.75,5.30) {$\hat{\mathbf{y}}$};

				\coordinate (pcSW) at (\XoutL,2.05);
				\coordinate (pcNE) at (\XoutR,3.80);
				\begin{scope}[on background layer]
					\node[testbox,fit=(pcSW)(pcNE)] (predbox) {};
				\end{scope}
				\node[testlabel,anchor=north west]
				at ([xshift=0.16cm,yshift=-0.10cm]predbox.north west)
				{Predicted Circuit};

				\predlogitinstB{12.20}{2.56}{\gs}
				\node[font=\scriptsize,text=black!70] at (12.20,{2.56+0.95*\gs+0.24})
				{$\hat{\mathbf{y}}$};

				\draw[trainflow] (training.east)  -- ([yshift=0.55cm]learner.west);
				\draw[testflow]  (inference.east) -- ([yshift=-0.55cm]learner.west);
				\draw[trainflow] ([yshift=0.55cm]learner.east)  -- (lossbox.west);
				\draw[testflow]  ([yshift=-0.55cm]learner.east) -- (predbox.west);

				\draw[backprop]
				(lossbox.north west)
				.. controls (9.25,6.15) and (8.45,5.95) .. ([xshift=0.50cm]learner.north)
				node[pos=0.5,sloped,above,inner sep=1.5pt,
						font=\scriptsize\itshape,text=red!60!black] {update $\theta$};

			\end{tikzpicture}
		}
		\caption{Existing circuit-localization methods infer a separate circuit for each model--task pair.
			The illustrated approach instead trains a shared graph learner $g_\theta$ across cases and applies it to unseen ones.
			Our example graph learner either processes the component-level computation graph directly or transforms it into a directed line graph~\citep{10.1007/BF02854581}.
			It then aligns heterogeneous features using TabPFN-style 1D feature attention~\citep{10.1038/s41586-024-08328-6}, processes the resulting graph with a directed graph learner such as DirGNN~\citep{10.48550/arXiv.2305.10498} (or DAGformer~\citep{10.52202/075280-2070}), and pools across token positions and prompt pairs using a permutation-invariant method~\citep{10.48550/arXiv.1703.06114}.
			Finally, it reads out one circuit-membership score for each edge in the original component-level computation graph.
			This architecture is one possible design.
			\Cref{app:framework} details each stage.
		}
		\label{fig:framework}
	\end{subfigure}
	\caption{Overview of (a) component-level computation graph construction and (b) per-case circuit localization versus\ cross-case circuit learning.
		\textbf{Takeaway:} GCL represents each model--task pair as a featured computation graph, so one trained graph learner can score the edges of unseen pairs.}
	\label{fig:overview}
\end{figure}

\section{Circuit Localization as a Graph Machine Learning Problem}
\label{sec:background}

For a trained transformer $\mathcal{M}$, we represent its computation as a component-level directed acyclic graph $G_{\mathcal{M}}=(V_{\mathcal{M}},E_{\mathcal{M}})$, whose nodes are components such as attention heads or multilayer perceptron (MLP) blocks and whose edges are possible computational pathways through the residual stream \citep{mathematical-framework,10.52202/075280-0719} (\Cref{fig:graph-construction} and \Cref{app:transformer_dag}).
A circuit is a sparse subgraph of $G_{\mathcal{M}}$ sufficient to reproduce a specified behavior.
For a model--task pair $(\mathcal{M},\mathcal{T})$, the task $\mathcal{T}$ provides clean--corrupted prompt pairs $(x,x')$ and a scalar metric $L_{\mathcal{T}}$.
Circuit localization infers a binary mask $\mathbf{y}\in\{0,1\}^{|E_{\mathcal{M}}|}$ over the candidate pathways: $y_e=1$ retains the clean message along edge $e$, whereas $y_e=0$ replaces it with the corresponding corrupted message.
We denote the metric under this masked computation by $L_{\mathcal{T}}(\mathcal{M},x,x';\mathbf{y})$.

At the component level, $G_{\mathcal{M}}$ is heterogeneous.
Its nodes have different computational roles, and their features vary in both dimension and semantics within and across models.
Therefore, learning across model--task pairs is more challenging than learning on a fixed graph with a shared feature space.
Similar challenges arise in heterogeneous graph foundation models~\citep{10.48550/arXiv.2505.15116}.
\Cref{app:feature-alignment} describes how we align features; \Cref{app:transformer_dag} discusses alternative graph representations.

The motivation for modeling relationships between computational pathways comes from their interactions.
Consider two edges $e_1$ and $e_2$, whose interventions induce fixed additive message changes $\Delta m_{e_1,p}$ and $\Delta m_{e_2,p}$ at the same read input $u_{t,p}$ of component $t$ at token position $p$.
Let $L_{t,p}(u)$ denote the scalar task metric obtained by replacing this read input with $u$ and continuing the downstream computation.
Writing $u_{t,p}^{\mathrm{cl}}$ for the clean read input, define $I_p(e_1,e_2)$ as the interaction contrast between applying both message changes jointly and applying each separately, as formalized in \Cref{eq:position-localised-interaction} of \Cref{app:theory}.
The intervention surface is the set of read inputs reached by applying the two message changes with independent strengths in $[0,1]$ (\Cref{eq:intervention-surface}).
Assume $L_{t,p}$ is twice continuously differentiable on a neighborhood containing this surface and that its Hessian is locally Lipschitz there.
Let $Q_{t,p}=\nabla_u^2 L_{t,p}(u_{t,p}^{\mathrm{cl}})$ denote the Hessian at the clean read input.
Then, for sufficiently small message changes, we have:
\begin{equation}
	I_p(e_1,e_2)
	=
	\Delta m_{e_1,p}^{\top}
	Q_{t,p}
	\Delta m_{e_2,p}
	+
	O\!\left(
	\left(
		\lVert\Delta m_{e_1,p}\rVert
		+
		\lVert\Delta m_{e_2,p}\rVert
		\right)^3
	\right).
	\label{eq:main-shared-target-interaction}
\end{equation}
\Cref{app:theory} contains the full derivation and proof.
A nonzero mixed term indicates that the leading local interaction depends jointly on both message changes and the downstream computation.
Thus, pathway relevance cannot always be decomposed into independent edge scores.
Graph representations that expose shared read inputs and other relationships among pathways provide graph learners with useful structure for modeling such interactions.

\section{Graph Machine Learning for Circuit Localization}
\label{sec:circuit-learning}

We investigate two complementary graph machine learning approaches for learning circuit masks.

\paragraph{Graph Circuit Learning (GCL)}

Rather than inferring a circuit mask independently for every model--task pair, GCL trains a  GNN across multiple model--task pairs using their ground-truth masks, then applies the learned predictor to unseen model--task pairs (\Cref{fig:framework}).

Each training case $(\mathcal{M}_i,\mathcal{T}_i)$ consists of a transformer with component-level computation graph $G_{\mathcal{M}_i}$, a task $\mathcal{T}_i$, and a ground-truth circuit mask $\mathbf{y}_i$.
We train one predictor across cases:
$
	g_\theta
	\left(
	G_{\mathcal{M}},
	\mathcal{T}
	\right)
	=
	\hat{\mathbf{y}}
	\in
	[0,1]^{|E_{\mathcal{M}}|}
$, 
where $\hat{y}_e$ estimates whether edge $e$ is in the ground-truth circuit, and training is end-to-end against the ground-truth masks with binary cross-entropy.
Further implementation details are provided in \Cref{app:framework}.

\Cref{fig:framework} illustrates our primary graph-learning architecture, which comprises five stages: graph representation, feature alignment, graph learning, pooling, and readout.
This architecture represents only one way to learn from the component-level computation graph.
Future work can explore alternative graph representations and learning architectures for circuit localization.

\paragraph{GNN Explainers for Circuit Localization}

PGExplainer \citep{10.48550/arXiv.2011.04573} is a GNN explainability approach that trains an MLP to generate edge masks explaining the predictions of a fixed, trained GNN.
In our adaptation of PGExplainer, the MLP uses fixed features derived from transformer parameters to score each edge in the component-level computation graph.
These scores define a sampled mask that is applied during a transformer forward pass.
The MLP is trained to preserve the clean model output.
A separate explainer is fitted for each model--task pair.
Figure~\ref{fig:pgexplainer-adaptation} illustrates how we adapt the original method from GNN explanation to circuit localization.

\begin{figure*}[t]
	\centering
		\resizebox{0.81\textwidth}{!}{%
			\begin{tikzpicture}[x=1.00cm,y=6.4cm,shift={(0,-0.5)},font=\scriptsize]
				\newcommand{\xtick}[2]{%
					\node[anchor=north,align=center] at (#1,0.4825) {#2};%
				}
				\newcommand{\learnedbox}[8]{%
					\draw[#7,line width=0.55pt] (#1,#2) -- (#1,#6);
					\draw[#7,line width=0.55pt] (#1-0.10,#2) -- (#1+0.10,#2);
					\draw[#7,line width=0.55pt] (#1-0.10,#6) -- (#1+0.10,#6);
					\filldraw[fill=#7!22,draw=#7,line width=0.65pt]
					(#1-0.22,#3) rectangle (#1+0.22,#5);
					\draw[#7,line width=1.05pt] (#1-0.22,#4) -- (#1+0.22,#4);
					\xtick{#1}{#8}
				}
				\newcommand{\statinterval}[4]{%
					\draw[#4!80!black,densely dotted,line width=0.4pt]
					({#1+0.29},{#2-#3}) -- ({#1+0.29},{#2+#3});
					\draw[#4!80!black,line width=0.4pt]
					({#1+0.245},{#2-#3}) -- ({#1+0.335},{#2-#3});
					\draw[#4!80!black,line width=0.4pt]
					({#1+0.245},{#2+#3}) -- ({#1+0.335},{#2+#3});
				}
				\newcommand{\learneduncertainties}[3]{%
					\foreach \value/\spread in {#2} {
							\statinterval{#1}{\value}{\spread}{#3}
						}
				}
				\newcommand{\baselinebox}[8]{%
					\draw[#7,densely dashed,line width=0.55pt] (#1,#2) -- (#1,#6);
					\draw[#7,densely dashed,line width=0.55pt] (#1-0.10,#2) -- (#1+0.10,#2);
					\draw[#7,densely dashed,line width=0.55pt] (#1-0.10,#6) -- (#1+0.10,#6);
					\filldraw[fill=#7!14,draw=#7,densely dashed,line width=0.65pt]
					(#1-0.22,#3) rectangle (#1+0.22,#5);
					\draw[#7,line width=1.05pt] (#1-0.22,#4) -- (#1+0.22,#4);
					\xtick{#1}{#8}
				}

				\foreach \y in {0.5,0.6,0.7,0.8,0.9,1.0} {
						\draw[black!12] (0.35,\y) -- (9.40,\y);
						\node[anchor=east] at (0.27,\y) {\pgfmathprintnumber[fixed,precision=1]{\y}};
					}
				\draw[->] (0.35,0.5) -- (0.35,1.0275);
				\node[rotate=90,anchor=south] at (-0.28,0.76) {Edge AUROC};
				\node[font=\small\bfseries] at (2.15,1.0375) {Our methods};
				\node[font=\small\bfseries] at (6.65,1.0375) {\textsc{InterpBench} baselines};
				\draw[black!35] (3.85,0.5) -- (3.85,1.01);
				\learnedbox{1.05}{0.710278}{0.840736}{0.946005}{0.979896}{0.990306}{custom2}{GCL\\Observed}
				\learneduncertainties{1.05}{0.710278/0.0489,0.840736/0.0438,0.946005/0.0138,0.979896/0.0068,0.990306/0.0066}{custom2}
				\learnedbox{2.15}{0.642222}{0.811283}{0.925844}{0.959339}{0.978000}{custom2}{GCL\\None}
				\learneduncertainties{2.15}{0.642222/0.0903,0.811283/0.0620,0.925844/0.0141,0.959339/0.0018,0.978000/0.0097}{custom2}
				\learnedbox{3.25}{0.743333}{0.881396}{0.949865}{0.987232}{1.000000}{custom5}{PGExplainer}
				\learneduncertainties{3.25}{0.743333/0.0518,0.881396/0.0352,0.949865/0.0086,0.987232/0.0012,1.000000/0.0000}{custom5}
				\learnedbox{4.45}{0.750000}{0.887045}{0.945339}{0.977250}{1.000000}{custom3}{ACDC}
				\learneduncertainties{4.45}{0.750000/0.0893,0.887045/0.0252,0.945339/0.0103,0.977250/0.0019,1.000000/0.0000}{custom3}
				\baselinebox{5.55}{0.819954}{0.938954}{0.972718}{0.995840}{1.000000}{custom1}{EAP-IG}
				\baselinebox{6.65}{0.705808}{0.901360}{0.941355}{0.972792}{1.000000}{custom4}{EAP}
				\learnedbox{7.75}{0.743750}{0.875926}{0.935380}{0.977728}{1.000000}{custom10}{Edge\\SP}
				\learneduncertainties{7.75}{0.743750/0.0219,0.875926/0.0162,0.935380/0.0038,0.977728/0.0048,1.000000/0.0000}{custom10}
				\learnedbox{8.85}{0.612245}{0.841052}{0.893083}{0.942045}{0.974821}{custom8}{Node\\SP}
				\learneduncertainties{8.85}{0.612245/0.0459,0.841052/0.0038,0.893083/0.0077,0.942045/0.0058,0.974821/0.0068}{custom8}
			\end{tikzpicture}
		}
	\caption{
		Edge AUROC across the 16 held-out cases of our evaluation split.
		For methods with five-seed results (solid box outlines), scores are averaged over 5 random seeds within each case; capped dotted intervals beside each box show $\pm 1$ seed-level standard deviation for each summary statistic.
		EAP and EAP-IG show negligible variability across seeds, so we omit their intervals (dashed box outlines).
		The highlighted GCL configuration uses DirGCN on the directed line graph with pooled query, key, and value supervision and was selected by cross-validation (\Cref{app:additional-results}).
		The ``None'' GCL control uses the same configuration with all message-passing connections removed.
		Cross-validation selects the control's hyperparameters and training duration separately.
		We reran the \textsc{InterpBench} baselines on the same evaluation split: ACDC~\citep{10.52202/075280-0719}, EAP~\citep{10.18653/v1/2024.blackboxnlp-1.25}, EAP-IG~\citep{10.48550/arXiv.2403.17806}, and edge-wise and node-wise Subnetwork Probing (Edge SP and Node SP)~\citep{10.18653/v1/2021.naacl-main.74}.
		\textbf{Takeaway:} GCL and our PGExplainer adaptation are both competitive with the rerun baselines, and removing message passing lowers the median edge AUROC of GCL from $0.946$ to $0.926$.}
	\label{fig:heldout-auroc}
\end{figure*}

\section{A Benchmark for Circuit Localization Across Model--Task Pairs}
\label{sec:benchmark}

Evaluating circuit-localization methods requires knowing which circuit is correct.
However, to our knowledge, there is no established way to obtain unique or provably correct ground-truth circuits for real-world transformers (\Cref{app:extended-related-work}).
\textsc{InterpBench}~\citep{10.52202/079017-2950} addresses this challenge using realistic semi-synthetic transformers with circuits known by construction.

Restricted Access Sequence Processing (RASP) is a domain-specific language for expressing sequence-processing algorithms in transformer-like operations~\citep{10.48550/arXiv.2106.06981}.
Each RASP-derived case begins with a program whose primitives correspond to transformer operations.
\textsc{Tracr} compiles the program into transformer weights~\citep{10.52202/075280-1649}, after which Strict Interchange Intervention Training (SIIT) trains a new transformer to preserve the same computation and ground-truth circuit while producing weight distributions closer to those of conventionally trained models~\citep{10.52202/079017-2950}.

However, \textsc{InterpBench} was designed for the per-model--task setting and contains too few distinct pairs to train a cross-case localizer or evaluate generalization to unseen pairs.
To support cross-case learning, we apply the same SIIT training procedure to 30 candidate \textsc{TracrBench} model--task pairs~\citep{10.48550/arXiv.2409.13714} and keep the 8 that pass our reconstruction check (\Cref{app:caseconstruction}).
We use 7 of the added pairs in the 25-case cross-validation pool and 1 in the 16-case held-out set, increasing the variety of source programs and circuits in both.
We also keep related cases together.
\Cref{app:benchmark} describes the added pairs and split design in detail.

\section{Preliminary Evidence and Discussion}
\label{sec:experiments}

We evaluate both methods by the area under the receiver operating characteristic curve over edges (edge AUROC) on each held-out pair in our benchmark (\Cref{app:experimental-details}).
We also report average precision and two behavior-based measures of circuit quality in \Cref{app:additional-experimental-results}.
We select GCL's configuration and training duration together with one of 18 hyperparameter settings using grouped 5-fold cross-validation on 25 cases and average results over 5 random seeds per held-out case.

\Cref{fig:heldout-auroc} shows the configuration selected by cross-validation from the 46 evaluated configurations, alongside our adapted PGExplainer and reruns of the \textsc{InterpBench} baselines.
The selected GCL configuration achieves a higher median edge AUROC ($0.946$) than the control ($0.926$).

Because the control retains every node and all node features, the comparison provides evidence that message passing over the directed line graph of the computation graph contributes to performance.
The pathway interactions in \Cref{sec:background} provide one motivation for graph learning, but the control does not establish that they are the only reason message passing helps.
More broadly, these results suggest that amortized graph learning is a promising approach to circuit localization, while the adapted PGExplainer’s performance demonstrates the feasibility of adapting graph explainability methods to this setting.
\Cref{app:additional-results} provides additional experimental results and ablations.
\Cref{app:threshold-auroc} reports edge AUROC with ROC curves over all score thresholds instead of the ten sparsity levels used here and compares the two evaluations.

\paragraph{Discussion.}
This preliminary study focuses on semi-synthetic model--task pairs with known ground-truth circuits, and our conclusions are limited to this setting.
Its aim is not to identify the optimal graph-learning architecture, but to show that circuit localization admits a natural graph-learning formulation and that this formulation \emph{can} improve performance in certain settings.
Future work should evaluate this approach on larger, real-world models, using metrics such as faithfulness when ground-truth circuits are unavailable \citep{10.48550/arXiv.2504.13151}, and explore the broader design space of graph representations and graph-learning methods for circuit localization.

\bibliographystyle{unsrtnat-preserve-title}
\bibliography{reference}

\appendix

\section{A Transformer as a Component-Level Directed Acyclic Graph (DAG)}
\label{app:transformer_dag}

In circuit localization, the computations inside a decoder-only transformer are often represented using a different, but mathematically equivalent, reformulation of the transformer definition given by \citet{10.48550/arXiv.1706.03762}.
This formulation is easier to interpret because the residual stream, through which all components communicate, is linear \citep{mathematical-framework}.
The nonlinear components, such as attention heads and MLP blocks, communicate through this residual stream via two operations:
each component (i) ``reads'' information from the residual stream through a projection $W_I^*$, and
(ii) ``writes'' information back into the residual stream by adding its internal output, projected through $W_O^*$ (see \Cref{fig:overview_residual_stream}).
The read and write projections constrain which directions in the residual stream a component can access and modify, and thus provide a useful geometric view of how components interact through the shared residual stream \citep{10.52202/079017-1962,10.48550/arXiv.2311.17030}.

\begin{figure}[htb]
	\centering
	\resizebox{\linewidth}{!}{%
		\begin{tikzpicture}[
				>=Stealth,
				stream/.style={draw=black, line width=2pt},
				read/.style={draw=custom3, line width=1.2pt, ->, rounded corners=6pt},
				write/.style={draw=custom4, line width=1.2pt, ->, rounded corners=6pt},
				internal/.style={draw=black, line width=1pt, ->},
				proj/.style={rectangle, draw=custom1!60, fill=custom1!10, thick, rounded corners=3pt, align=center, font=\footnotesize},
				head/.style={rectangle, draw=custom2!80, fill=custom2!10, thick, rounded corners=3pt, align=center, font=\footnotesize},
				comp/.style={rectangle, draw=custom5!80, fill=custom5!10, thick, rounded corners=3pt, align=center, font=\footnotesize},
				add/.style={circle, draw=black, thick, minimum size=0.7cm, inner sep=0pt, font=\Large},
				label/.style={text=black}
			]

			\node (x0) [font=\Large] {$x_0$};
			\node (add_a1) [add, right=2.0cm of x0] {$+$};
			\node (add_m1) [add, right=2.5cm of add_a1] {$+$};
			\node (x1) [right=0.5cm of add_m1] {};
			\node (xn1) [right=1.0cm of x1] {};
			\node (add_an) [add, right=2.0cm of xn1] {$+$};
			\node (add_mn) [add, right=2.5cm of add_an] {$+$};
			\node (xn) [font=\Large, right=1.0cm of add_mn] {$x_n$};

			\draw[stream] (x0) -- (add_a1);
			\draw[stream] (add_a1) -- (add_m1);
			\draw[stream] (add_m1) -- (x1) node[above=0.05cm, label, font=\bfseries] {$x_1$};
			\draw[stream, dotted] (x1) -- (xn1) node[above=0.05cm, label, font=\bfseries] {$x_{n-1}$};
			\draw[stream] (xn1) -- (add_an);
			\draw[stream] (add_an) -- (add_mn);
			\draw[stream, ->] (add_mn) -- (xn);

			\foreach \l/\leftA/\leftM in {1/x0/add_a1, n/xn1/add_an} {

					\coordinate (mid_a\l) at ($(\leftA)!0.7!(add_a\l)$);
					\coordinate (r_a\l) at ($(\leftA.east) + (0.5, 0)$);

					\node (h\l_1) [head, above right=1.3cm and 0.7cm of mid_a\l] {$A_{\l}^1$};
					\node (h\l_2) [head, above=0.8cm of h\l_1] {$A_{\l}^2$};

					\node (wi_a\l_1) [proj, left=0.3cm of h\l_1] {$W_I^{A^1_\l}$};
					\node (wo_a\l_1) [proj, right=0.3cm of h\l_1] {$W_O^{A^1_\l}$};

					\node (wi_a\l_2) [proj, left=0.3cm of h\l_2] {$W_I^{A^2_\l}$};
					\node (wo_a\l_2) [proj, right=0.3cm of h\l_2] {$W_O^{A^2_\l}$};

					\draw[read] (mid_a\l)
					-- ++(0, 0.8)
					-| ([xshift=-0.5cm]wi_a\l_1.west)
					-- (wi_a\l_1.west);

					\draw[read] (mid_a\l)
					-- ++(0, 0.8)
					-| ([xshift=-0.5cm]wi_a\l_2.west) %
					-- (wi_a\l_2.west);

					\draw[internal] (wi_a\l_1) -- (h\l_1);
					\draw[internal] (h\l_1) -- (wo_a\l_1);
					\draw[internal] (wi_a\l_2) -- (h\l_2);
					\draw[internal] (h\l_2) -- (wo_a\l_2);

					\draw[write] (wo_a\l_1.east)
					-- ++(0.4, 0)
					|- ([yshift=0.8cm]add_a\l.center)
					-- (add_a\l.north);

					\draw[write] (wo_a\l_2.east)
					-- ++(0.4, 0) %
					|- ([yshift=0.8cm]add_a\l.center)
					-- (add_a\l.north);

					\coordinate (mid_m\l) at ($(\leftM)!0.5!(add_m\l)$);
					\coordinate (r_m\l) at ($(\leftM.east) + (0.5, 0)$);

					\node (mlp\l) [comp, below left=1.2cm and 0.2cm of mid_m\l] {$\text{MLP}_\l$};
					\node (wi_m\l) [proj, left=0.3cm of mlp\l] {$W_I^{\text{MLP}_\l}$};
					\node (wo_m\l) [proj, right=0.3cm of mlp\l] {$W_O^{\text{MLP}_\l}$};

					\draw[read] (mid_m\l)
					-- ++(0, -0.6)
					-| ([xshift=-0.4cm]wi_m\l.west)
					-- (wi_m\l.west);

					\draw[internal] (wi_m\l) -- (mlp\l);
					\draw[internal] (mlp\l) -- (wo_m\l);

					\draw[write] (wo_m\l.east) -| (add_m\l.south);
				}

			\begin{scope}[on background layer]
				\foreach \l in {1, n} {
						\node[draw=gray!40, fill=gray!5, thick, rounded corners=10pt, dashed,
							fit=(r_a\l) (wi_a\l_2) (wo_a\l_2) (wi_m\l) (wo_m\l) (add_m\l), inner sep=0.2cm] (layer\l_group) {};
						\node[above right, font=\large, text=black!60] at (layer\l_group.north west) {Layer $\l$};
					}
			\end{scope}

		\end{tikzpicture}
	}
	\caption{Residual-stream view of an $n$-layer decoder-only transformer.
		Attention heads (yellow) and MLP blocks (pink) read from the residual stream through input projections (blue) and write additive updates back through output projections (red).
		\textbf{Takeaway:} Components interact only through the residual stream, so their reads and writes define the edges of the component-level computation graph.}
	\label{fig:overview_residual_stream}
\end{figure}

We can equivalently unroll the transformer's residual-stream representation into a component-level computation graph, a directed acyclic graph (DAG) in which each node is a model component and directed edges represent possible activation pathways from upstream writes to downstream reads \citep{10.52202/075280-0719,10.52202/079017-0587} (\Cref{fig:overview_computational_DAG}).
Because the residual stream is additive, every downstream component can in principle read information written by any earlier component.
This induces edges not only between adjacent layers, but also between non-adjacent components connected through the residual stream.
The resulting component-level DAG is not intended as an efficient implementation of the transformer, but as a faithful abstraction for reasoning about direct and mediated information flow between components.

Under this abstraction, an edge from an upstream component $*_i$ to a downstream component $*_j$ can be associated with an effective linear map, or ``virtual weight'' \citep{mathematical-framework}.
Using column-vector notation, this virtual weight is given by
\begin{equation}
	W_I^{*_j} W_O^{*_i},
\end{equation}
where $W_O^{*_i}$ is the write projection of the upstream component and $W_I^{*_j}$ is the read projection of the downstream component.
This virtual weight describes how directions written in the residual stream by component $*_i$ are visible to the input subspace read by component $*_j$.
Thus, virtual weights provide a parameter-space explanation for why edges in the component-level computation graph are meaningful, while circuit-localization methods typically intervene on the corresponding activations along these edges.

The virtual weight covers only the residual-stream segment between an upstream write and a downstream read projection; component-internal nonlinearities, such as MLP activations and attention softmax, remain inside the nodes.
Its linearity therefore requires that nothing nonlinear acts on the residual stream between writes and reads.
This holds exactly for the models studied here: all benchmark transformers are \textsc{Tracr}-derived models trained with SIIT \citep{10.52202/079017-2950}, and they inherit the \textsc{Tracr}-compiled architecture, which contains no normalization layers \citep{10.52202/075280-1649}.
In typical transformers, every read is instead preceded by LayerNorm, whose scale and bias can be folded into the read projection but whose per-input rescaling remains nonlinear, so virtual weights there describe write--read interactions only up to an input-dependent rescaling \citep{mathematical-framework}.
Circuit-localization methods treat normalization as part of the downstream component, which keeps the component-level DAG exact; only the linear form of the edge map is specific to models without normalization layers.

\begin{figure}[htb]
	\centering
	\adjustbox{trim=0pt 17pt 0pt 21pt,clip,width=\linewidth}{%
		\begin{tikzpicture}[
				>=Stealth,
				stream/.style={draw=black!10, line width=2pt},
				read/.style={draw=custom3!20, line width=1.2pt, ->, rounded corners=6pt},
				write/.style={draw=custom4!20, line width=1.2pt, ->, rounded corners=6pt},
				internal/.style={draw=black!10, line width=1pt, ->},
				proj/.style={rectangle, draw=custom1!20, fill=custom1!5, thick, rounded corners=3pt, align=center, font=\footnotesize, text=black!20},
				add/.style={circle, draw=black!15, fill=white, thick, minimum size=0.7cm, inner sep=0pt, font=\Large, text=black!15},
				label/.style={text=black!20},
				head/.style={rectangle, draw=custom2!80, fill=custom2!10, thick, rounded corners=3pt, align=center, font=\footnotesize},
				comp/.style={rectangle, draw=custom5!80, fill=custom5!10, thick, rounded corners=3pt, align=center, font=\footnotesize},
				virtual/.style={draw=black!75, line width=1pt, ->},
				vlabel/.style={font=\tiny, text=black!100, fill=custom6!15, fill opacity=0.9, text opacity=1, inner sep=1.5pt, rounded corners=2pt, draw=custom6!100, dashed, line width=0.5pt}
			]

			\node (x0) [font=\Large\bfseries] {$x_0$};
			\node (add_a1) [add, right=2.0cm of x0] {$+$};
			\node (add_m1) [add, right=2.5cm of add_a1] {$+$};
			\node (x1) [right=0.5cm of add_m1] {};
			\node (xn1) [right=1.0cm of x1] {};
			\node (add_an) [add, right=2.0cm of xn1] {$+$};
			\node (add_mn) [add, right=2.5cm of add_an] {$+$};
			\node (xn) [font=\Large\bfseries, right=1.0cm of add_mn] {$x_n$};

			\draw[stream] (x0) -- (add_a1);
			\draw[stream] (add_a1) -- (add_m1);
			\draw[stream] (add_m1) -- (x1) node[above=0.05cm, label, font=\bfseries] {$x_1$};
			\draw[stream, dotted] (x1) -- (xn1) node[above=0.05cm, label, font=\bfseries] {$x_{n-1}$};
			\draw[stream] (xn1) -- (add_an);
			\draw[stream] (add_an) -- (add_mn);
			\draw[stream, ->] (add_mn) -- (xn);

			\foreach \l/\leftA/\leftM in {1/x0/add_a1, n/xn1/add_an} {

					\coordinate (mid_a\l) at ($(\leftA)!0.7!(add_a\l)$);
					\coordinate (r_a\l) at ($(\leftA.east) + (0.5, 0)$);

					\node (h\l_1) [head, above right=1.3cm and 0.7cm of mid_a\l] {$A_{\l}^1$};
					\node (h\l_2) [head, above=0.8cm of h\l_1] {$A_{\l}^2$};

					\node (wi_a\l_1) [proj, left=0.3cm of h\l_1] {$W_I^{A^1_\l}$};
					\node (wo_a\l_1) [proj, right=0.3cm of h\l_1] {$W_O^{A^1_\l}$};

					\node (wi_a\l_2) [proj, left=0.3cm of h\l_2] {$W_I^{A^2_\l}$};
					\node (wo_a\l_2) [proj, right=0.3cm of h\l_2] {$W_O^{A^2_\l}$};

					\draw[read] (mid_a\l) -- ++(0, 0.8) -| ([xshift=-0.5cm]wi_a\l_1.west) -- (wi_a\l_1.west);
					\draw[read] (mid_a\l) -- ++(0, 0.8) -| ([xshift=-0.5cm]wi_a\l_2.west) -- (wi_a\l_2.west);
					\draw[internal] (wi_a\l_1) -- (h\l_1); \draw[internal] (h\l_1) -- (wo_a\l_1);
					\draw[internal] (wi_a\l_2) -- (h\l_2); \draw[internal] (h\l_2) -- (wo_a\l_2);
					\draw[write] (wo_a\l_1.east) -- ++(0.4, 0) |- ([yshift=0.8cm]add_a\l.center) -- (add_a\l.north);
					\draw[write] (wo_a\l_2.east) -- ++(0.4, 0) |- ([yshift=0.8cm]add_a\l.center) -- (add_a\l.north);

					\coordinate (mid_m\l) at ($(\leftM)!0.5!(add_m\l)$);
					\coordinate (r_m\l) at ($(\leftM.east) + (0.5, 0)$);

					\node (mlp\l) [comp, below left=1.2cm and 0.2cm of mid_m\l] {$\text{MLP}_\l$};
					\node (wi_m\l) [proj, left=0.3cm of mlp\l] {$W_I^{\text{MLP}_\l}$};
					\node (wo_m\l) [proj, right=0.3cm of mlp\l] {$W_O^{\text{MLP}_\l}$};

					\draw[read] (mid_m\l) -- ++(0, -0.6) -| ([xshift=-0.4cm]wi_m\l.west) -- (wi_m\l.west);
					\draw[internal] (wi_m\l) -- (mlp\l); \draw[internal] (mlp\l) -- (wo_m\l);
					\draw[write] (wo_m\l.east) -| (add_m\l.south);
				}

			\begin{scope}[on background layer]
				\foreach \l in {1, n} {
						\node[draw=gray!20, fill=gray!2, thick, rounded corners=10pt, dashed,
							fit=(r_a\l) (wi_a\l_2) (wo_a\l_2) (wi_m\l) (wo_m\l) (add_m\l), inner sep=0.2cm] (layer\l_group) {};
						\node[above right, font=\large, text=black!30] at (layer\l_group.north west) {Layer $\l$};
					}
			\end{scope}

			\draw[virtual] (h1_2.east) to[out=20, in=105] node[vlabel, pos=0.5] {$W_U W_O^{A^2_1}$} (xn.north);
			\draw[virtual] (hn_2.east) to[out=-10, in=170] node[vlabel, pos=0.5] {$W_U W_O^{A^2_n}$} (xn.west);
			\draw[virtual] (h1_1.east) to[out=45, in=120] node[vlabel, pos=0.7] {$W_U W_O^{A^1_1}$} (xn.north west);
			\draw[virtual] (hn_1.east) to[out=-20, in=190] node[vlabel, pos=0.65] {$W_U W_O^{A^1_n}$} (xn.west);
			\draw[virtual] (mlp1.east) to[out=-20, in=-150] node[vlabel, pos=0.5] {$W_U W_O^{\text{MLP}_1}$} (xn.south west);
			\draw[virtual] (mlpn.east) to[out=15, in=-150] node[vlabel, pos=0.5] {$W_U W_O^{\text{MLP}_n}$} (xn.west);

			\draw[virtual] (x0.south) to[out=-60, in=-120, looseness=0.6] node[vlabel, pos=0.5, inner sep=3pt] {$W_U W_E$} (xn.south);

			\draw[virtual] (h1_1.south) to[out=-105, in=120, looseness=1.5] node[vlabel, pos=0.5] {$W_I^{\text{MLP}_1} W_O^{A^1_1}$} (mlp1.north);
			\draw[virtual] (h1_2.south) to[out=-60, in=60, looseness=1.2] node[vlabel, pos=0.5] {$W_I^{\text{MLP}_1} W_O^{A^2_1}$} (mlp1.north);

			\draw[virtual] ([yshift=-4pt]x0.north) to[out=80, in=160] ([yshift=-4pt]hn_2.north west);
			\draw[virtual] ([yshift=-2pt]x0.north) to[out=80, in=160] ([yshift=-2pt]hn_2.north west);
			\draw[virtual] (x0.north) to[out=80, in=160] node[vlabel, pos=0.5] {$W_I^{A^2_n} W_E$} (hn_2.north west);

			\draw[virtual] ([yshift=2pt]x0.north) to[out=50, in=160] ([yshift=2pt]hn_1.west);
			\draw[virtual] ([yshift=-2pt]x0.north) to[out=50, in=160] ([yshift=-2pt]hn_1.west);
			\draw[virtual] (x0.north) to[out=50, in=160] node[vlabel, pos=0.3] {$W_I^{A^1_n} W_E$} (hn_1.west);

			\draw[virtual] ([yshift=2pt]x0.north) to[out=45, in=180] ([yshift=2pt]h1_1.west);
			\draw[virtual] ([yshift=-2pt]x0.north) to[out=45, in=180] ([yshift=-2pt]h1_1.west);
			\draw[virtual] (x0.north) to[out=45, in=180] node[vlabel, pos=0.5] {$W_I^{A^1_1} W_E$} (h1_1.west);

			\draw[virtual] ([yshift=2pt]x0.north) to[out=60, in=180] ([yshift=2pt]h1_2.west);
			\draw[virtual] ([yshift=-2pt]x0.north) to[out=60, in=180] ([yshift=-2pt]h1_2.west);
			\draw[virtual] (x0.north) to[out=60, in=180] node[vlabel, pos=0.5] {$W_I^{A^2_1} W_E$} (h1_2.west);

			\draw[virtual] (x0.south) to[out=-45, in=180] node[vlabel, pos=0.5] {$W_I^{\text{MLP}_1} W_E$} (mlp1.west);
			\draw[virtual] (x0.south) to[out=-15, in=165] node[vlabel, pos=0.65] {$W_I^{\text{MLP}_n} W_E$} (mlpn.west);

			\draw[virtual] ([yshift=2pt]h1_2.east) to[out=20, in=160] ([yshift=2pt]hn_2.west);
			\draw[virtual] ([yshift=-2pt]h1_2.east) to[out=20, in=160] ([yshift=-2pt]hn_2.west);

			\draw[virtual] ([yshift=2pt]h1_1.east) to[out=30, in=170] ([yshift=2pt]hn_2.west);
			\draw[virtual] ([yshift=-2pt]h1_1.east) to[out=30, in=170] ([yshift=-2pt]hn_2.west);
			\draw[virtual] (h1_1.east) to[out=30, in=170] node[vlabel, pos=0.35] {$W_I^{A^2_n} W_O^{A^1_1}$} (hn_2.west);

			\draw[virtual] ([yshift=2pt]mlp1.east) to[out=45, in=200] ([yshift=2pt]hn_2.west);
			\draw[virtual] ([yshift=-2pt]mlp1.east) to[out=45, in=200] ([yshift=-2pt]hn_2.west);
			\draw[virtual] (mlp1.east) to[out=45, in=200] node[vlabel, pos=0.25] {$W_I^{A^2_n} W_O^{\text{MLP}_1}$} (hn_2.west);

			\draw[virtual] ([yshift=2pt]h1_2.east) to[out=10, in=135] ([yshift=2pt]hn_1.west);
			\draw[virtual] ([yshift=-2pt]h1_2.east) to[out=10, in=135] ([yshift=-2pt]hn_1.west);
			\draw[virtual] (h1_2.east) to[out=10, in=135] node[vlabel, pos=0.65] {$W_I^{A^1_n} W_O^{A^2_1}$} (hn_1.west);
			\draw[virtual] (h1_2.east) to[out=20, in=160] node[vlabel, pos=0.35] {$W_I^{A^2_n} W_O^{A^2_1}$} (hn_2.west);

			\draw[virtual] ([yshift=2pt]h1_1.east) to[out=20, in=160] ([yshift=2pt]hn_1.west);
			\draw[virtual] ([yshift=-2pt]h1_1.east) to[out=20, in=160] ([yshift=-2pt]hn_1.west);
			\draw[virtual] (h1_1.east) to[out=20, in=160] node[vlabel, pos=0.65] {$W_I^{A^1_n} W_O^{A^1_1}$} (hn_1.west);

			\draw[virtual] ([yshift=2pt]mlp1.east) to[out=20, in=190] ([yshift=2pt]hn_1.west);
			\draw[virtual] ([yshift=-2pt]mlp1.east) to[out=20, in=190] ([yshift=-2pt]hn_1.west);
			\draw[virtual] (mlp1.east) to[out=20, in=190] node[vlabel, pos=0.75] {$W_I^{A^1_n} W_O^{\text{MLP}_1}$} (hn_1.west);

			\draw[virtual] (mlp1.east) to[out=0, in=180] node[vlabel, pos=0.5] {$W_I^{\text{MLP}_n} W_O^{\text{MLP}_1}$} (mlpn.west);
			\draw[virtual] (h1_1.east) to[out=0, in=170] node[vlabel, pos=0.35] {$W_I^{\text{MLP}_n} W_O^{A^1_1}$} (mlpn.west);
			\draw[virtual] (h1_2.east) to[out=-10, in=160] node[vlabel, pos=0.35] {$W_I^{\text{MLP}_n} W_O^{A^2_1}$} (mlpn.west);

			\draw[virtual] (hn_1.south) to[out=-105, in=120, looseness=1.5] node[vlabel, pos=0.5] {$W_I^{\text{MLP}_n} W_O^{A^1_n}$} (mlpn.north);
			\draw[virtual] (hn_2.south) to[out=-60, in=60, looseness=1.2] node[vlabel, pos=0.5] {$W_I^{\text{MLP}_n} W_O^{A^2_n}$} (mlpn.north);

		\end{tikzpicture}%
	}
	\caption{Component-level directed acyclic graph induced by the residual stream.
		Edges represent possible residual-stream-mediated activation pathways from upstream writes to downstream reads.
		Attention heads have separate query, key, and value read pathways.
		\textbf{Takeaway:} A circuit is a subgraph of this directed acyclic graph, so circuit localization selects a subset of its edges.}
	\label{fig:overview_computational_DAG}
\end{figure}
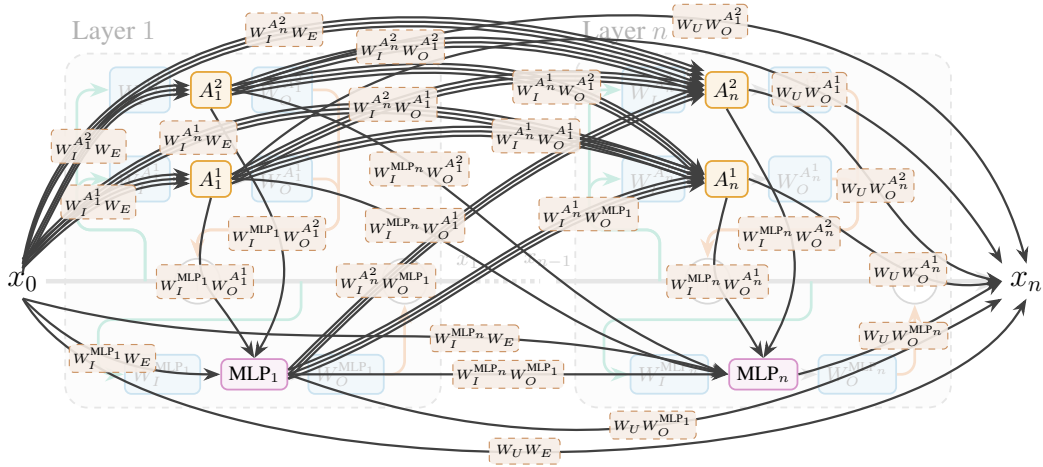

The definition of the read projection $W_I^*$ and write projection $W_O^*$ depends on the component class of $*$.
For MLPs of the general form
\begin{equation}
	\mathrm{MLP}(\mathbf{x}) = W_{\mathrm{out}} h(W_{\mathrm{in}}\mathbf{x}),
\end{equation}
where arbitrary hidden layers and nonlinear activations are abstracted into $h$, the input matrix $W_{\mathrm{in}} \in \mathbb{R}^{d_{\mathrm{mlp}} \times d_{\mathrm{model}}}$ and output matrix $W_{\mathrm{out}} \in \mathbb{R}^{d_{\mathrm{model}} \times d_{\mathrm{mlp}}}$ act as the read and write projections, respectively:
\begin{equation}
	W_I^{\mathrm{MLP}} = W_{\mathrm{in}},
	\qquad
	W_O^{\mathrm{MLP}} = W_{\mathrm{out}}.
\end{equation}

Attention heads write to the residual stream through a single output projection $W_O^A$, which plays the same role for a head as $W_{\mathrm{out}}$ does for an MLP, but read from the stream through three distinct projections that form their queries, keys, and values \citep{mathematical-framework,10.52202/079017-1962}.
Therefore, an upstream component $*_i$ can connect to a downstream attention head $A_j$ through three distinct read pathways, governed by the virtual weights
\begin{align}
	\text{Value pathway:} \quad & W_V^{A_j} W_O^{*_i}, \\
	\text{Query pathway:} \quad & W_Q^{A_j} W_O^{*_i}, \\
	\text{Key pathway:} \quad   & W_K^{A_j} W_O^{*_i}.
\end{align}
These read-side virtual weights should be distinguished from the standard head-internal QK and OV circuits of the Transformer Circuits framework: under the column-vector convention used here, $W_Q^{\top} W_K$ determines how an attention head routes information across token positions, while $W_O W_V$ determines what information the head writes back into the residual stream \citep{mathematical-framework}.

Lastly, input tokens are projected into the residual stream via the embedding matrix $W_E \in \mathbb{R}^{d_{\mathrm{model}} \times d_{\mathrm{vocab}}}$, and the final residual stream state is transformed back to logits via the unembedding matrix $W_U \in \mathbb{R}^{d_{\mathrm{vocab}} \times d_{\mathrm{model}}}$.
All bias terms are omitted from these formulations for simplicity.
For details on folding biases into augmented weight matrices, see~\citet{10.48550/arXiv.2511.20273}.

Because \textsc{InterpBench} and the Mechanistic Interpretability Benchmark (MIB) formulate circuit localization over model components and their connections, we use the corresponding component-level computation graph in this preliminary study \citep{10.52202/079017-2950,10.48550/arXiv.2504.13151}.
Other graph representations are possible, such as neuron-level graphs that treat individual neurons as nodes and their connections as edges.
This representation avoids matrix edge features and heterogeneity in node features.
It is, however, a much larger graph, which can exceed billions of edges even for a relatively small language model.
It is also directed and carries signed edge weights.
Large size, edge direction, and signed weights together make the neuron-level graph a challenging and underexplored setting for graph learning \citep{10.48550/arXiv.2202.10793,10.48550/arXiv.2209.00546,3CcP3VI6LW,10.48550/arXiv.2602.04768}.
We leave the study of graph representations beyond the component-level computation graph to future work.

\section{The Graph Circuit Learning Framework in Detail}
\label{app:framework}

This appendix specifies the Graph Circuit Learning architectures used in our experiments.

For a model--task pair $(\mathcal{M},\mathcal{T})$, let $G_{\mathcal{M}}=(V_{\mathcal{M}},E_{\mathcal{M}})$ be the component-level computation graph defined in \Cref{app:transformer_dag}, and let $\mathbf{y}\in\{0,1\}^{|E_{\mathcal{M}}|}$ be its ground-truth circuit mask.
All architecture variants use the same raw feature construction and produce one logit for each edge in the original order of $E_{\mathcal{M}}$.
They differ in their graph representation, graph-learning architecture, query/key/value supervision rule, and connectivity condition.

\subsection{Features}
\label{app:feature-construction}

A task $\mathcal{T}$ supplies a collection of clean--corrupted prompt pairs $(x_b,x'_b)$.
Let $\mathcal{B}_{\mathcal{T}}$ denote the clean--corrupted prompt pairs used to construct features for a model--task pair.
For each pair $b\in\mathcal{B}_{\mathcal{T}}$, let $\mathcal{P}_b$ be the non-padding token positions, and let $c=(b,p)$ denote the context for prompt pair $b$ at position $p$.
The component-level computation graph and ground-truth mask are fixed within a model--task pair, whereas the activations and gradients vary across contexts.
The implementation encodes each context $(b,p)$ separately and pools the resulting edge representations only after graph learning.

Separate clean and corrupted forward passes provide component activations.
A backward pass on the clean computation provides gradients of a task-specific scalar feature-extraction loss, denoted by $\ell_{\mathcal{T}}^{\mathrm{feat}}$, specified by the benchmark for each model--task pair.
The feature extractor uses KL divergence for tasks with discrete answer distributions, $L_1$ loss for floating-point targets, and, for tasks whose batches expose only class labels rather than full target distributions, the benchmark's logit-difference loss.
For component $v$, let $a_{v,c}^{\mathrm{cl}}$ and $a_{v,c}^{\mathrm{co}}$ denote its clean and corrupted vectors at that component's input or output interface, and let
\begin{equation}
	g_{v,c}
	=
	\nabla_{a_{v,c}^{\mathrm{cl}}} \ell_{\mathcal{T}}^{\mathrm{feat}}
	\label{eq:component-gradient}
\end{equation}
denote the corresponding clean-computation gradient.
For each component, we use the vector at the part of the computation represented by that component: attention writers use the per-head vector before the output projection, attention readers use the query, key, or value vector, MLP writers use the hidden vector after the nonlinearity, and MLP readers use the vector before the nonlinearity.

The forward pass records component states rather than a separate vector for every edge in the component-level computation graph.
For an edge $e=(s,t)$, the map from the source component's vector to that edge's contribution to the target component's input is
\begin{equation}
	W_e
	=
	W_I^t W_O^s,
	\qquad
	W_e:
	\mathbb{R}^{d_s}
	\rightarrow
	\mathbb{R}^{d_t},
	\label{eq:virtual-edge-operator}
\end{equation}
under the column-vector convention of \Cref{app:transformer_dag}.
The implementation retains the corresponding read and write factors and applies them successively, so it does not need to materialize every dense matrix $W_e$.
The clean and corrupted pathway messages are
\begin{equation}
	m_{e,c}^{\mathrm{cl}}
	=
	W_e a_{s,c}^{\mathrm{cl}},
	\qquad
	m_{e,c}^{\mathrm{co}}
	=
	W_e a_{s,c}^{\mathrm{co}}.
	\label{eq:edge-messages}
\end{equation}
These are the clean and corrupted edge messages denoted by $z_e^{\mathrm{cl}}$ and $z_e^{\mathrm{co}}$ in \Cref{fig:graph-construction}.

To calculate the clean gradient for each edge, we represent the target component's input by the sum of clean messages from its incoming edges, $u_{t,c}^{\mathrm{lin}}=\sum_{e'=(s',t)}m_{e',c}^{\mathrm{cl}}$.
Since each clean message contributes additively, $\partial u_{t,c}^{\mathrm{lin}}/\partial m_{e,c}^{\mathrm{cl}}=I$ and
\begin{equation}
	\nabla_{m_{e,c}^{\mathrm{cl}}}\ell_{\mathcal{T}}^{\mathrm{feat}}
	=
	g_{t,c}.
	\label{eq:message-gradient}
\end{equation}
To express this gradient with respect to the source component's vector, we compute
\begin{equation}
	\widetilde{g}_{e,c}
	=
	W_e^{\top}g_{t,c}.
	\label{eq:transported-gradient}
\end{equation}
The gradient $g_{t,c}$ is the same for all edges entering $t$, whereas $\widetilde{g}_{e,c}$ also depends on the map $W_e$ for that edge.

For graph learning, we use the following six features for every edge $e=(s,t)$ and each context $c$:
\begin{equation}
	a_{s,c}^{\mathrm{cl}},
	\quad
	a_{s,c}^{\mathrm{co}},
	\quad
	g_{t,c},
	\quad
	m_{e,c}^{\mathrm{cl}},
	\quad
	m_{e,c}^{\mathrm{co}},
	\quad
	\widetilde{g}_{e,c}.
	\label{eq:computation-graph-edge-features}
\end{equation}
The first three vectors are the source component's clean and corrupted states and the gradient at the target component.
The final three are the clean and corrupted contributions along the edge and the target gradient mapped back through that edge's linear map.

These features have a simple transformation law under a compatible reparameterization of the transformer by a change of basis at each component interface, following the general parameter-space symmetries studied for neural networks and transformers \citep{10.48550/arXiv.2603.10090,10.48550/arXiv.2506.22712}.
Let $A_v$ be the invertible change of coordinates for component $v$.
If states transform as $a_{v,c}\mapsto A_v a_{v,c}$ and gradients transform as $g_{v,c}\mapsto A_v^{-\top}g_{v,c}$, then the edge operator transforms as $W_e\mapsto A_tW_eA_s^{-1}$ for $e=(s,t)$.
The six features therefore transform as
\begin{equation}
	\begin{aligned}
		a_{s,c}^{\mathrm{cl}},a_{s,c}^{\mathrm{co}}
		 & \mapsto
		A_s a_{s,c}^{\mathrm{cl}},A_s a_{s,c}^{\mathrm{co}},
		 &
		g_{t,c}
		 & \mapsto
		A_t^{-\top}g_{t,c}, \\
		m_{e,c}^{\mathrm{cl}},m_{e,c}^{\mathrm{co}}
		 & \mapsto
		A_t m_{e,c}^{\mathrm{cl}},A_t m_{e,c}^{\mathrm{co}},
		 &
		\widetilde{g}_{e,c}
		 & \mapsto
		A_s^{-\top}\widetilde{g}_{e,c}.
	\end{aligned}
	\label{eq:edge-feature-transport}
\end{equation}
Thus, the feature vectors are covariant representations of component coordinates.
Their covariant roles make basis-invariant scalar combinations available:
\begin{equation}
	\begin{aligned}
		g_{t,c}^{\top}m_{e,c}^{\mathrm{cl}}
		 & =
		\widetilde{g}_{e,c}^{\top}a_{s,c}^{\mathrm{cl}},
		 &
		g_{t,c}^{\top}m_{e,c}^{\mathrm{co}}
		 & =
		\widetilde{g}_{e,c}^{\top}a_{s,c}^{\mathrm{co}}.
	\end{aligned}
	\label{eq:invariant-edge-feature-contractions}
\end{equation}
The graph learner could learn to use such invariant combinations from the raw features.

\subsection[Graph Representations]{Graph Representations}
\label{app:graph-transformations}

Because the prediction target is an edge mask, both representations produce one classifiable object for every original component-level edge.
\paragraph{Original component graph.}
The first representation is the original component-level computation graph $G_{\mathcal{M}}$, without a graph transformation.
For each context $c$, component $v$ receives its clean activation $a_{v,c}^{\mathrm{cl}}$, corrupted activation $a_{v,c}^{\mathrm{co}}$, and gradient $g_{v,c}$.
After graph learning, GCL scores each component edge from the ordered representations of its source and target together with the six edge features in \Cref{eq:computation-graph-edge-features}.

\paragraph{Directed line graph.}
Let $\bar e$ denote the transformed node corresponding to $e\in E_{\mathcal{M}}$.
The directed line graph \citep{10.1007/BF02854581} has node set
\begin{equation}
	V_{\mathcal{L}}
	=
	\{\bar e:e\in E_{\mathcal{M}}\}.
\end{equation}
For $e_1=(s_1,t_1)$ and $e_2=(s_2,t_2)$, the implementation treats $e_1$ and $e_2$ as composable if either $t_1=s_2$ or $t_1$ is an internal read node and $s_2$ is the corresponding write node of the same attention head or MLP block.
The second case is required because the component-level computation graphs in \textsc{InterpBench} represent an attention head's query, key, or value read separately from its output write, and likewise represent an MLP's input read separately from its output write.
Writing this implemented relation as $e_1\prec e_2$, the directed line-graph adjacency is
\begin{equation}
	E_{\mathcal{L}}^{\mathrm{obs}}
	=
	\{(\bar e_1,\bar e_2):e_1\prec e_2\}.
	\label{eq:implemented-line-adjacency}
\end{equation}
Thus, the directed line graph has $|E_{\mathcal{M}}|$ nodes, and circuit localization becomes binary node classification.
For each context, line-graph node $\bar e$ receives the six features for edge $e$ in \Cref{eq:computation-graph-edge-features}.

The two representations expose different relationships among these candidate pathways.
Both representations are equivariant to a consistent reindexing of the component-level computation graph.
For a component reindexing $\sigma$, the induced edge reindexing maps $e=(s,t)$ to $\sigma(e)=(\sigma(s),\sigma(t))$.
The directed line graph then maps $\bar e$ to $\overline{\sigma(e)}$.
The original computation graph maps component node $v$ to $\sigma(v)$ and edge $e$ to $\sigma(e)$.
In each case, the adjacency, attached features, component edges, and circuit labels are relabeled in the same way.

\subsection{Feature Alignment}
\label{app:feature-alignment}

The six features defined in \Cref{eq:computation-graph-edge-features} naturally have variable widths and semantics as different transformers typically have different hidden dimensions and learn different hidden representations.
We therefore use a feature alignment module, adapted from the official TabPFN implementation \citep{10.1038/s41586-024-08328-6}, to map every feature to a fixed-width representation of dimension $d_{\mathrm{align}}$.

Let
\begin{equation}
	q^{(\rho)}
	=
	\left(
	q_1^{(\rho)},\ldots,q_{d_\rho}^{(\rho)}
	\right)
	\in
	\mathbb{R}^{d_\rho}
\end{equation}
be a feature vector with role $\rho$ and coordinate identifiers $\kappa_1,\ldots,\kappa_{d_\rho}$.
Each scalar becomes one fixed-width feature token
\begin{equation}
	z_j^{(\rho)}
	=
	A_{\mathrm{val}}q_j^{(\rho)}
	+
	b_{\mathrm{val}}
	+
	c_{\psi}(\kappa_j)
	+
	r(\rho).
	\label{eq:feature-token}
\end{equation}
Here, $A_{\mathrm{val}}q_j^{(\rho)}+b_{\mathrm{val}}$ is a learned scalar-value projection and $r(\rho)$ is a learned role embedding.
The coordinate embedding $c_{\psi}(\kappa_j)$ assigns each graph-wide coordinate identifier a Gaussian random vector and maps that random vector into a fixed-dimensional space through a trainable projection.
The six role embeddings $r(\rho)$, one for each of the six features, are learned from random initialization.
A summary token is appended to the feature-token sequence, and its transformed representation is the fixed-width aligned representation for that feature:
\begin{equation}
	F_{\psi}
	\left(
	q^{(\rho)};
	\rho,
	\boldsymbol{\kappa}
	\right)
	\in
	\mathbb{R}^{d_{\mathrm{align}}}.
	\label{eq:feature-alignment}
\end{equation}
On the directed line graph, concatenating the six aligned roles in \Cref{eq:computation-graph-edge-features} gives the initial representation $\phi_{\bar e,c}$ of line-graph node $\bar e$.
On the original computation graph, concatenating the aligned clean activation, corrupted activation, and gradient gives the initial representation $\phi_{v,c}$ of component $v$.
Concatenating all six aligned roles also gives each component edge an ordered representation $q_{e,c}$ for the final edge readout.

\subsection{Graph Learning Models}
\label{app:graph-learning}

For each graph node $v$ and context $c=(b,p)$, let $\phi_{v,c}$ denote the fixed-width representation formed by aligning and concatenating the features attached to that node.
Let $\bar{\mathcal{G}}_{\mathcal{M}}$ denote the selected graph representation, either the original computation graph or its directed line graph.
For each context $c=(b,p)$, the graph learner processes one copy of $\bar{\mathcal{G}}_{\mathcal{M}}$ with its context-specific features and returns one hidden state for every graph node:
\begin{equation}
	\left\{
	h_{v,b,p}
	\right\}_{v\in V(\bar{\mathcal{G}}_{\mathcal{M}})}
	=
	\operatorname{GraphLearner}_{\theta}
	\left(
	\bar{\mathcal{G}}_{\mathcal{M}},
	\left\{
	\phi_{v,b,p}
	\right\}_{v\in V(\bar{\mathcal{G}}_{\mathcal{M}})}
	\right).
	\label{eq:context-graph-learning}
\end{equation}
The parameters are shared across edges, clean--corrupted prompt pairs, positions, graph sizes, and training cases.
For the directed line graph, we test DirGNN with the line-graph connections in \Cref{eq:implemented-line-adjacency} and with no connections.
For the original computation graph, we test DirGNN and DAGformer with the observed connections, with no connections, and with all possible connections.
When we remove connections, the model still scores the same component edges, but those edges cannot exchange information through message passing.
With all possible connections, we connect every ordered pair of distinct components for DirGNN and let every component attend to every component for DAGformer.

\paragraph{DirGNN}
Our first graph learning backend uses PyTorch Geometric's \citep{10.48550/arXiv.1903.02428,10.48550/arXiv.2507.16991} DirGNN \citep{10.48550/arXiv.2305.10498} construction with one of four aggregators: GraphConv \citep{10.1609/aaai.v33i01.33014602}, GCN \citep{10.48550/arXiv.1609.02907}, GraphSAGE \citep{10.48550/arXiv.1706.02216}, or GAT \citep{10.48550/arXiv.1710.10903}.
We call the four models DirGraphConv, DirGCN, DirSAGE, and DirGAT.
Each model additively aggregates incoming neighbors in the observed direction and in the reversed direction, and combines the two neighbor results with equal weight $\alpha=0.5$.
In each direction, DirGCN scales the message along an edge from node $u$ to node $v$ by $1/\sqrt{d_{\mathrm{out}}(u)\,d_{\mathrm{in}}(v)}$.
Both degrees count only the edges in that direction.
DirGCN adds no self-loops.
Following the block design of Scalable Message Passing Neural Networks \citep{10.48550/arXiv.2411.00835}, each layer applies the directed graph-convolution update and a per-node feed-forward transformation in separate pre-LayerNorm residual branches, with SiLU activations and learned residual scales.

\paragraph{DAGformer}
Our second graph learning backend uses DAGformer, a transformer design for directed acyclic graphs \citep{10.52202/075280-2070}.

We use DAGformer only with the original computation graph.
With the observed connections, the attention mask is the symmetrized transitive closure of that graph as in DAGformer, and we use DAGformer's positional encoding.
With no connections, no component attends to any other component.
With all possible connections, every component attends to every component, including itself.
Both controls omit the positional encoding.

\subsection{Pooling, Readout, and Training}
\label{app:training}

After graph learning, we obtain a representation $h_{e,b,p}$ for every edge $e\in E_{\mathcal{M}}$, retained clean--corrupted prompt pair $b$, and non-padding token position $p\in\mathcal{P}_b$.
For the directed line graph, $h_{e,b,p}$ is the state of line-graph node $\bar e$.
For the original computation graph, a shared MLP first combines the ordered representations of the source and target of component edge $e=(s,t)$ and then concatenates the six aligned edge features:
\begin{equation}
	u_{e,b,p}
	=
	\operatorname{MLP}_{\mathrm{end}}
	\left([h_{s,b,p}\,\|\,h_{t,b,p}]\right),
	\qquad
	h_{e,b,p}
	=
	[u_{e,b,p}\,\|\,q_{e,b,p}].
	\label{eq:raw-edge-readout}
\end{equation}

We use Deep Sets \citep{10.48550/arXiv.1703.06114} to pool the context-specific representations for each edge.
We first pool representations across the non-padding token positions, then pool the result across the clean--corrupted prompt pairs used to construct the features.
We write $h_{i,e}$ for the representation of edge $e$ in case $i$ after both pooling stages.

A shared linear readout maps each pooled edge representation to one circuit-membership logit
\begin{equation}
	\ell_{i,e}
	=
	w_{\mathrm{out}}^{\top}h_{i,e}
	+
	b_{\mathrm{out}},
	\qquad
	\hat y_{i,e}
	=
	\sigma(\ell_{i,e}).
	\label{eq:circuit-readout}
\end{equation}

For training case $i$, let $C_i\subseteq E_{\mathcal{M}_i}$ be the set of edges in the ground-truth circuit.
For each edge $e\in E_{\mathcal{M}_i}$, let $y_{i,e}=1$ when $e\in C_i$ and $y_{i,e}=0$ otherwise.
\textsc{InterpBench}'s circuit-preparation code removes the edges excluded by the evaluator and maps the retained query, key, and value inputs to their parent attention-head connection.
We study two supervision rules.
The separate rule retains each query, key, and value input as a supervision unit and gives it the label of its parent connection.
The pooled rule takes the maximum of their logits and applies the loss once to the parent connection.
Validation, model selection, and held-out evaluation all calculate AUROC at the attention-head connection level.
Let $\mathcal{U}_i$ denote the supervision units in case $i$ under the selected rule.
For $u\in\mathcal{U}_i$, let $\ell_{i,u}$ and $y_{i,u}$ be its logit and binary target.
Let $m_i=|\mathcal{U}_i|$, $n_i^{\mathrm{circuit}}=\sum_{u\in\mathcal{U}_i}y_{i,u}$, and $n_i^{\mathrm{other}}=m_i-n_i^{\mathrm{circuit}}$, and define
\begin{equation}
	w_i^{\mathrm{circuit}}
	=
	\max
	\left(
	1,
	\frac{n_i^{\mathrm{other}}}{\max(n_i^{\mathrm{circuit}},1)}
	\right).
	\label{eq:circuit-class-weight}
\end{equation}
The implemented binary cross-entropy-with-logits objective can equivalently be written as
\begin{equation}
	\begin{aligned}
		\mathcal{L}_i
		 & =
		-\frac{1}{m_i}
		\sum_{u\in\mathcal{U}_i}
		\Bigl[
		\\
		 & \qquad
			w_i^{\mathrm{circuit}}y_{i,u}\log\sigma(\ell_{i,u})
		\\
		 & \qquad
			+
			(1-y_{i,u})\log\left(1-\sigma(\ell_{i,u})\right)
			\Bigr].
	\end{aligned}
	\label{eq:training-objective}
\end{equation}
The mean is therefore taken over the selected supervision units, and positive units receive extra weight only when they are fewer than negative units.

\section{Adapting PGExplainer to Circuit Localization}
\label{app:pgexplainer}

Figure~\ref{fig:pgexplainer-adaptation} illustrates our adaptation of PGExplainer \citep{10.48550/arXiv.2011.04573} from its original GNN explainability application to circuit localization.
In both versions, the MLP is trained so that the fixed predictor's output on the masked input matches its unmasked output.
We fit a separate explainer for each model--task pair, using only that pair's prompts and model outputs (\Cref{app:experimental-details}).

\begin{figure}[H]
	\centering
	\resizebox{\linewidth}{!}{%
		\begin{tikzpicture}[
				>=Stealth,
				font=\small,
				panel/.style={
						rounded corners=6pt,
						draw=black!55,
						fill=black!2,
						inner sep=5pt
					},
				input/.style={
						rounded corners=4pt,
						draw=custom1!85!black,
						fill=custom1!8,
						align=center,
						minimum width=2.55cm,
						minimum height=1.0cm
					},
				explainer/.style={
						rounded corners=4pt,
						draw=custom2!85!black,
						fill=custom2!10,
						align=center,
						minimum width=1.55cm,
						minimum height=1.0cm
					},
				mask/.style={
						rounded corners=4pt,
						draw=custom5!85!black,
						fill=custom5!10,
						align=center,
						minimum width=1.75cm,
						minimum height=1.0cm
					},
				model/.style={
						rounded corners=4pt,
						draw=custom3!85!black,
						fill=custom3!10,
						align=center,
						minimum width=2.25cm,
						minimum height=1.0cm
					},
				output/.style={
						rounded corners=4pt,
						draw=custom4!85!black,
						fill=custom4!8,
						align=center,
						minimum width=1.6cm,
						minimum height=0.8cm
					},
				note/.style={
						rounded corners=3pt,
						draw=custom8!85,
						fill=custom8!8,
						dashed,
						align=center,
						minimum width=2.1cm,
						minimum height=0.7cm
					},
				flow/.style={->, thick, draw=black!70},
				feedback/.style={->, thick, dashed, draw=custom4!85!black}
			]

			\node[panel, minimum width=14cm, minimum height=4.3cm] (top-panel) at (7.5, 1.4) {};
			\node[anchor=west, font=\bfseries] at (0.75, 3.1) {PGExplainer for graph neural networks};

			\node[input] (gnn-input) at (2.15, 1.75) {input graph $G$\\and\\node representations\\from the trained GNN};
			\node[explainer] (gnn-mlp) at (4.95, 1.75) {MLP};
			\node[mask] (gnn-mask) at (7.4, 1.75) {probability of\\keeping each edge};
			\node[model] (gnn-model) at (10.35, 1.75) {GNN with sampled\\edge mask};
			\node[output] (gnn-output) at (13.15, 1.75) {prediction with\\edge mask};
			\node[note] (gnn-loss) at (11.15, 0.45) {preserve the\\original prediction};

			\draw[flow] (gnn-input) -- (gnn-mlp);
			\draw[flow] (gnn-mlp) -- (gnn-mask);
			\draw[flow] (gnn-mask) -- (gnn-model);
			\draw[flow] (gnn-model) -- (gnn-output);
			\draw[feedback] (gnn-output.south) |- (gnn-loss.east);
			\draw[feedback] (gnn-loss.south) -- ++(0, -0.45) -| (gnn-mlp.south);

			\node[panel, minimum width=14cm, minimum height=4.4cm] (bottom-panel) at (7.5, -3.25) {};
			\node[anchor=west, font=\bfseries] at (0.75, -1.55) {Adaptation of PGExplainer for circuit localization};

			\node[input] (transformer-input) at (2.15, -2.85) {computation graph\\and\\fixed description of\\the transformation\\on each edge};
			\node[explainer] (transformer-mlp) at (4.95, -2.85) {MLP};
			\node[mask] (transformer-mask) at (7.4, -2.85) {probability of\\keeping each edge};
			\node[model] (transformer-model) at (10.25, -2.85) {transformer with\\sampled edge mask\\on clean inputs};
			\node[output] (transformer-output) at (13.15, -2.85) {prediction with\\edge mask};
			\node[note] (transformer-loss) at (11.15, -4.35) {preserve the clean\\model output};
			\node[note, minimum width=2.1cm] (corrupted) at (8.1, -4.35) {use corrupted\\activations when\\an edge is masked};
			\node[note, minimum width=2.25cm] (labels) at (2.0, -4.70) {ground-truth circuit\\not used to train\\this MLP};

			\draw[flow] (transformer-input) -- (transformer-mlp);
			\draw[flow] (transformer-mlp) -- (transformer-mask);
			\draw[flow] (transformer-mask) -- (transformer-model);
			\draw[flow] (transformer-model) -- (transformer-output);
			\draw[flow] (corrupted.north) -- (transformer-model.south);
			\draw[feedback] (transformer-output.south) |- (transformer-loss.east);
			\draw[feedback] (transformer-loss.south) -- ++(0,-0.3) -| (transformer-mlp.south);

			\node[font=\scriptsize, text=black!65, align=center] at (7.5, -5.85) {Solid arrows show the calculation.
				Dashed arrows show how the MLP is trained.};
		\end{tikzpicture}%
	}
	\caption{
		Comparison of the original PGExplainer procedure and our adaptation to circuit localization.
		In the upper panel, PGExplainer~\citep{10.48550/arXiv.2011.04573} scores each graph edge using representations produced by a trained GNN for the edge's two endpoint nodes.
		In the lower panel, our adaptation scores each computation-graph edge using a fixed representation of its associated linear transformation.
		In both settings, the scores define a sampled edge mask applied while the fixed predictor runs.
		In the transformer adaptation, masking an edge replaces its contribution with the corresponding corrupted activation.
		\textbf{Takeaway:} Unlike GCL, our PGExplainer adaptation needs no ground-truth circuits to fit, and instead fits a separate explainer to each model--task pair.}
	\label{fig:pgexplainer-adaptation}
\end{figure}
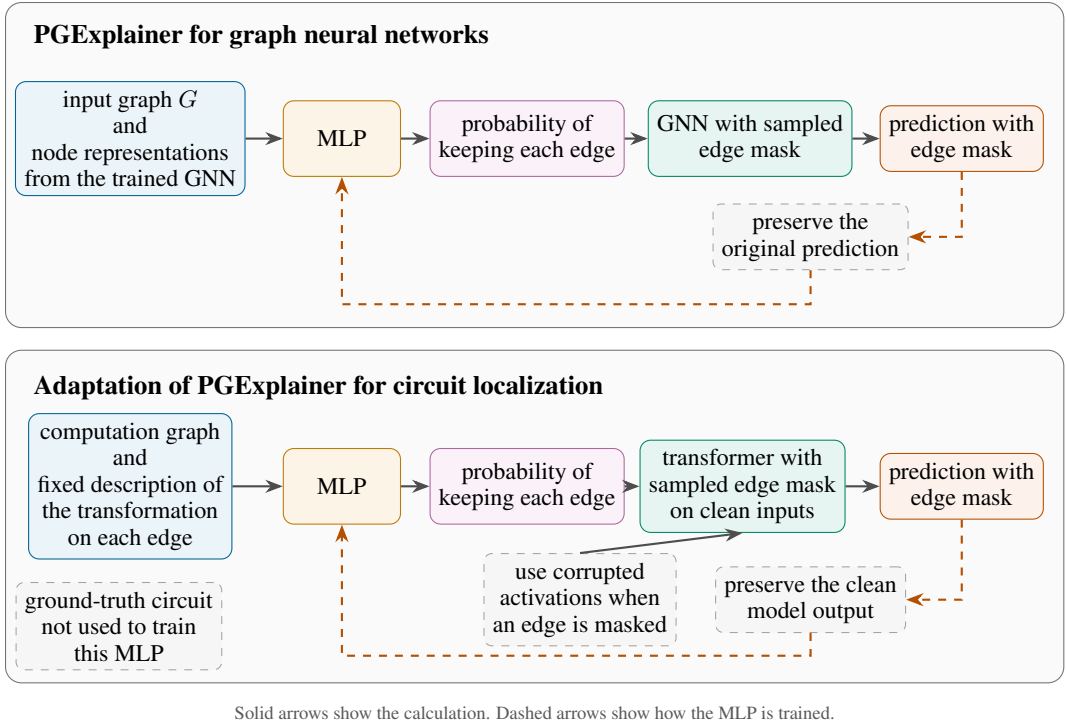

\section{An Augmented \textsc{InterpBench} Benchmark for Circuit Localization Across Model--Task Pairs}
\label{app:benchmark}

\subsection{Creating Additional Model--Task Pairs}
\label{app:caseconstruction}

\textsc{TracrBench} pairs RASP programs with input--output examples for sequence-processing tasks \citep{10.48550/arXiv.2409.13714}.
It does not provide the low-level transformers trained with Strict Interchange Intervention Training (SIIT) or the query/key/value-separated connection masks required by our \textsc{InterpBench}-based circuit-localization setting.
For each candidate task, we compile its RASP program with \textsc{Tracr} to obtain a high-level transformer and its ground-truth circuit.
We then train a corresponding low-level transformer through \textsc{InterpBench}'s SIIT procedure so that it implements the same computation while matching \textsc{InterpBench}'s architectural conventions and producing parameter distributions closer to those of conventionally trained models.
We convert the resulting ground-truth circuit to the same query/key/value-separated connection-level definition used for \textsc{InterpBench} cases.

We reload each trained transformer from its saved weights and check that it reproduces the source program's behavior on held-out inputs.
Of the 30 candidates, we exclude 21 whose trained transformers fail this check and 1 whose program does not compile within the check's time limit.
We use this procedure to create the 8 additional pairs reported in \Cref{sec:benchmark}.

\subsection[Evaluation Split That Limits Leakage Across Model--Task Pairs]{Evaluation Split That Limits Leakage Across Model--Task Pairs}
\label{app:benchmark-split}

To evaluate transfer to unseen model--task pairs, we group related pairs and assign each group entirely to either the training set or the evaluation set within each split.
We treat two model--task pairs as related when they match under at least one of the following checks:
\begin{enumerate}
	\item \textbf{RASP program templates.}
	      After standardizing function, helper, argument, and local names; removing docstrings, decorators, annotations, and type comments; and replacing string and numeric literal values with placeholders, two pairs match when their program templates are identical.
	\item \textbf{Source-program structure.}
	      Two pairs match when their source programs have the same normalized control and expression structure after task-specific surface details are removed.
	\item \textbf{Certified semantic equality.}
	      Two pairs match when program analysis certifies that they behave identically.
	\item \textbf{Equal circuits after evaluation mapping.}
	      Two pairs match when their circuits are equal after applying \textsc{InterpBench}'s query/key/value mapping.
	\item \textbf{Role-preserving circuit isomorphism.}
	      Two pairs match when a directed graph isomorphism preserves each component's computational role.
\end{enumerate}
We treat these matches as transitive, so any two pairs joined by a sequence of matches belong to the same group.

We form the split from 84 \textsc{InterpBench} pairs and the 8 added \textsc{TracrBench} pairs described in \Cref{app:caseconstruction}.
\textsc{InterpBench} also contains two indirect-object-identification cases (IOI and IOI next-token), which are implemented directly as transformer models rather than compiled from RASP programs.
We use neither, so 84 of its 86 model--task pairs enter the pool.
The held-out set contains 15 \textsc{InterpBench} cases and one added \textsc{TracrBench} case in 16 groups.
The cross-validation set contains 18 \textsc{InterpBench} cases and 7 added \textsc{TracrBench} cases in 7 groups.
The added cases increase the variety of source programs and circuits in both sets.
We exclude the remaining 51 pairs.
We exclude one \textsc{InterpBench} pair because its released model is defective.
We exclude another 48 \textsc{InterpBench} pairs as duplicates.
After \textsc{InterpBench}'s query/key/value mapping, they have the same circuit as 8 retained pairs.
We exclude the last 2 pairs because each belongs to a group that contains a held-out pair.
We assign each cross-validation group as a whole to one validation fold and use the other four folds for training.
\Cref{tab:graph-corpus-split} reports the size of each fold and of the held-out test set, how many of their pairs come from \textsc{InterpBench} and from \textsc{TracrBench}, and how many groups of related pairs each contains.

\begin{table}[htbp]
	\centering
	\caption{
		Model--task pairs used for 5-fold cross-validation and the held-out test set.
		Each validation fold is used once for validation, while the other four folds are used for training.
		\textbf{Takeaway:} Each group of related pairs lies entirely in the cross-validation pool or entirely in the held-out test set, so no held-out pair matches a training pair under any of the five checks.}
	\label{tab:graph-corpus-split}
	\resizebox{\textwidth}{!}{%
		\begin{tabular}{lrrrr}
			\toprule
			Subset                & Model--task pairs & Original \textsc{InterpBench} pairs & \textsc{TracrBench}-derived pairs & Groups of related pairs \\
			\midrule
			Cross-validation pool & 25                & 18                                  & 7                                 & 7                       \\
			Validation fold 0     & 8                 & 3                                   & 5                                 & 1                       \\
			Validation fold 1     & 5                 & 5                                   & 0                                 & 1                       \\
			Validation fold 2     & 4                 & 3                                   & 1                                 & 1                       \\
			Validation fold 3     & 4                 & 3                                   & 1                                 & 2                       \\
			Validation fold 4     & 4                 & 4                                   & 0                                 & 2                       \\
			Held-out test set     & 16                & 15                                  & 1                                 & 16                      \\
			\bottomrule
		\end{tabular}%
	}
\end{table}

\Cref{tab:graph-corpus-component-sizes,tab:graph-corpus-ground-truth-circuit-sizes} describe the component-level computation graphs and ground-truth circuits represented in the cross-validation pool and held-out test set, in terms of their numbers of nodes and edges.

The five checks above keep the specific kinds of program and circuit similarity they identify from crossing between the held-out test set and the cross-validation pool.
They do not give a complete theoretical account of when two model--task pairs should be treated as related; developing one is an important direction for future work.

\begin{table*}[!h]
	\centering
	\footnotesize
	\setlength{\tabcolsep}{3pt}
	\caption{
		Sizes of the component-level computation graphs in each validation fold and in the held-out test set.
		Each row gives the minimum, median, and maximum across the model--task pairs in that validation fold or in the held-out test set.
		\textbf{Takeaway:} The largest held-out graph has 128 nodes and 1{,}520 edges, whereas the largest cross-validation graph has 56 nodes and 262 edges.}
	\label{tab:graph-corpus-component-sizes}
	\begin{tabular}{@{}l|rrr|rrr@{}}
		\toprule
		   & \multicolumn{3}{c|}{\# of component nodes} & \multicolumn{3}{c}{\# of directed component edges}  \\
		\cmidrule(lr){2-4}\cmidrule(l){5-7}
		Set                        & Min & ~~~Median & ~~~~~~Max & Min & ~~~~~~Median & Max  \\
		\midrule
		Fold 0 ($n=8$)             & 38 & 38  & 38  & 110 & 110 & 110  \\
		Fold 1 ($n=5$)             & 38 & 38  & 38  & 110 & 110 & 110  \\
		Fold 2 ($n=4$)             & 56 & 56  & 56  & 262 & 262 & 262  \\
		Fold 3 ($n=4$)             & 38 & 38  & 38  & 110 & 110 & 110  \\
		Fold 4 ($n=4$)             & 38 & 38  & 38  & 110 & 110 & 110  \\
		Held-out test set ($n=16$) & 38 & 56  & 128 & 110 & 262 & 1520 \\
		\bottomrule
	\end{tabular}
\end{table*}

\begin{table*}[!h]
	\centering
	\footnotesize
	\setlength{\tabcolsep}{3pt}
	\caption{
		Sizes of the ground-truth circuits in each validation fold and in the held-out test set.
		Each row gives the minimum, median, and maximum across the model--task pairs in that validation fold or in the held-out test set.
		We count circuit nodes and edges before \textsc{InterpBench} removes edges and merges query, key, and value connections for evaluation.
		\textbf{Takeaway:} The largest held-out circuit has 43 nodes and 66 edges, whereas the largest cross-validation circuit has 21 nodes and 27 edges.}
	\label{tab:graph-corpus-ground-truth-circuit-sizes}
	\begin{tabular}{@{}l|rrr|rrr@{}}
		\toprule
		  & \multicolumn{3}{c|}{\# of circuit nodes} & \multicolumn{3}{c}{\# of circuit edges} \\
		\cmidrule(lr){2-4}\cmidrule(l){5-7}
    		Set  & Min & Median  & Max & Min & Median  & Max \\
		\midrule
		Fold 0 ($n=8$)             & 2  & 4  & 4   & 1   & 3   & 3   \\
		Fold 1 ($n=5$)             & 10 & 11  & 11  & 10  & 11  & 11  \\
		Fold 2 ($n=4$)             & 21 & 21 & 21  & 25  & 27  & 27  \\
		Fold 3 ($n=4$)             & 10  & 12  & 12  & 10  & 11.5 & 12  \\
		Fold 4 ($n=4$)             & 7  & 10 & 13  & 6   & 11   & 16  \\
		Held-out test set ($n=16$) & 4 & 16   & 43  & 2   & 18   & 66  \\
		\bottomrule
	\end{tabular}
\end{table*}

\section{Shared-Target Edge Interactions}
\label{app:theory}

Here, we motivate the graph-based design of GCL through an analysis of pairwise edge interactions.

\subsection{Pairwise Edge Interactions}

Circuit-localization methods commonly assign one scalar importance score to each edge.
Such scores may contain contextual information, but they do not explicitly represent whether the effect of one edge intervention depends on another.

Fix a clean--corrupted prompt pair $(x,x')\sim\mathcal{T}$.
For a set of jointly patched edges $S\subseteq E_{\mathcal{M}}$, let
$\mathbf{y}^{(S)}\in\{0,1\}^{|E_{\mathcal{M}}|}$
denote the mask with $y_e^{(S)}=0$ if $e\in S$ and $y_e^{(S)}=1$ otherwise.
Using the masked task metric defined in \Cref{sec:background}, the signed effect of jointly patching these edges is
\begin{equation}
	\delta_S
	=
	L_{\mathcal{T}}
	\Bigl(
	\mathcal{M},
	x,
	x';
	\mathbf{y}^{(S)}
	\Bigr)
	-
	L_{\mathcal{T}}
	\Bigl(
	\mathcal{M},
	x,
	x';
	\mathbf{1}
	\Bigr).
	\label{eq:joint-intervention-effect}
\end{equation}
For two edges, we define their interaction as the difference between their joint effect and the sum of their individual effects:
\begin{equation}
	I(e_1,e_2)
	=
	\delta_{\{e_1,e_2\}}
	-
	\delta_{\{e_1\}}
	-
	\delta_{\{e_2\}}.
	\label{eq:edge-interaction}
\end{equation}
If $I(e_1,e_2)=0$, the joint intervention effect equals the sum of the individual effects.
A nonzero value indicates that the effect of patching one edge changes when the other edge is patched as well.

Such non-additivity can arise through different mechanisms.
For composable edges, patching a downstream edge may overwrite or screen off an effect propagated from an upstream intervention.
For causally parallel edges, neither intervention changes the message naturally produced by the other, but their effects may still interact through subsequent nonlinear computation.
We analyze a particularly clear parallel setting in which two messages are added at the same component input.

\subsection{Interactions at a Shared Read Input}

Fix a token position $p$ in the prompt pair $(x,x')$.
Consider two edges
\begin{equation}
	e_1=(s_1,t),
	\qquad
	e_2=(s_2,t)
\end{equation}
whose messages enter the same read input of component $t$ at this position.
For each edge $e_i$, let
\begin{equation}
	\Delta m_{e_i,p}
	=
	m_{e_i,p}^{\mathrm{co}}
	-
	m_{e_i,p}^{\mathrm{cl}}
	\label{eq:shared-target-message-change}
\end{equation}
denote the change in its transmitted message under patching.
The construction of these messages is described in \Cref{app:feature-construction}.

Because both messages enter the same read input, the clean input to component $t$ can be written as
\begin{equation}
	u_{t,p}^{\mathrm{cl}}
	=
	m_{e_1,p}^{\mathrm{cl}}
	+
	m_{e_2,p}^{\mathrm{cl}}
	+
	r_{t,p}^{\mathrm{cl}},
	\label{eq:shared-read-input}
\end{equation}
where $r_{t,p}^{\mathrm{cl}}$ contains all remaining contributions to this input.
This decomposition assumes that the read-side transformations represented by the component-level computation graph are included consistently in the edge messages.
Any remaining read-side nonlinearity is treated as part of the downstream computation.

Patching $e_1$, $e_2$, or both therefore changes the shared input to
\begin{equation}
	u_{t,p}^{\mathrm{cl}}+\Delta m_{e_1,p},
	\qquad
	u_{t,p}^{\mathrm{cl}}+\Delta m_{e_2,p},
	\qquad
	u_{t,p}^{\mathrm{cl}}
	+\Delta m_{e_1,p}
	+\Delta m_{e_2,p},
\end{equation}
respectively.
This setting is mathematically convenient because the two interventions are fixed additive changes to the same variable.

To isolate the consequences of changing this shared input, let $L_{t,p}(u)$ denote the task metric obtained by replacing the read input of component $t$ at position $p$ by $u$ and continuing the downstream computation normally.
Thus, $L_{t,p}$ is the original task metric viewed as a function of one internal model input.

Let $I_p(e_1,e_2)$ denote the interaction between the corresponding interventions localized to position $p$.
$I_p$ is the analogue of \Cref{eq:edge-interaction} for interventions restricted to the position-$p$ read input; patching an edge in \Cref{eq:edge-interaction} changes read inputs at every position, so we study the position-localized quantity directly rather than deriving it from $I$.
Using the three possible patched inputs above, this interaction is:
\begin{align}
	I_p(e_1,e_2)
	={} &
	L_{t,p}
	\left(
	u_{t,p}^{\mathrm{cl}}
	+
	\Delta m_{e_1,p}
	+
	\Delta m_{e_2,p}
	\right)
	\nonumber \\
	    & -
	L_{t,p}
	\left(
	u_{t,p}^{\mathrm{cl}}
	+
	\Delta m_{e_1,p}
	\right)
	-
	L_{t,p}
	\left(
	u_{t,p}^{\mathrm{cl}}
	+
	\Delta m_{e_2,p}
	\right)
	\nonumber \\
	    & +
	L_{t,p}
	\left(
	u_{t,p}^{\mathrm{cl}}
	\right).
	\label{eq:position-localised-interaction}
\end{align}
This expression quantifies the difference between the joint effect of the two message changes and the sum of their individual effects.

We assume that $L_{t,p}$ is twice continuously differentiable in a neighborhood containing the intervention surface defined in \Cref{eq:intervention-surface} and that its Hessian is locally Lipschitz.
Next, we derive an exact Hessian representation of \Cref{eq:position-localised-interaction}, from which the local approximation in \Cref{prop:shared-target} follows.

\subsection{Exact Interaction Identity}

For compactness, let
\begin{equation}
	d_1
	=
	\Delta m_{e_1,p},
	\qquad
	d_2
	=
	\Delta m_{e_2,p},
	\qquad
	u_0
	=
	u_{t,p}^{\mathrm{cl}}.
\end{equation}
We interpolate the strength of the two interventions independently using $s,r\in[0,1]$ and define
\begin{equation}
	\Psi(s,r)
	=
	L_{t,p}
	\left(
	u_0
	+
	s d_1
	+
	r d_2
	\right).
	\label{eq:intervention-surface}
\end{equation}
The four corners of this intervention surface represent the clean computation, the two individual interventions, and the joint intervention:
\begin{align}
	\Psi(0,0)
	 & =
	L_{t,p}(u_0),     \\
	\Psi(1,0)
	 & =
	L_{t,p}(u_0+d_1), \\
	\Psi(0,1)
	 & =
	L_{t,p}(u_0+d_2), \\
	\Psi(1,1)
	 & =
	L_{t,p}(u_0+d_1+d_2).
\end{align}
It follows from \Cref{eq:position-localised-interaction} that
\begin{equation}
	I_p(e_1,e_2)
	=
	\Psi(1,1)
	-
	\Psi(1,0)
	-
	\Psi(0,1)
	+
	\Psi(0,0).
	\label{eq:interaction-surface}
\end{equation}

Applying the fundamental theorem of calculus with respect to $r$ gives
\begin{align}
	\Psi(1,1)-\Psi(1,0)
	 & =
	\int_0^1
	\frac{\partial\Psi}{\partial r}(1,r)
	\,\mathrm{d}r, \\
	\Psi(0,1)-\Psi(0,0)
	 & =
	\int_0^1
	\frac{\partial\Psi}{\partial r}(0,r)
	\,\mathrm{d}r.
\end{align}
Subtracting these expressions and applying the fundamental theorem of calculus with respect to $s$ yields
\begin{equation}
	I_p(e_1,e_2)
	=
	\int_0^1
	\int_0^1
	\frac{\partial^2\Psi}{\partial s\,\partial r}(s,r)
	\,\mathrm{d}s\,\mathrm{d}r.
	\label{eq:integrated-mixed-partial}
\end{equation}
By the chain rule,
\begin{equation}
	\frac{\partial^2\Psi}{\partial s\,\partial r}(s,r)
	=
	d_1^\top
	\nabla_u^2 L_{t,p}
	\left(
	u_0+s d_1+r d_2
	\right)
	d_2.
\end{equation}
Consequently, the interaction has the exact representation
\begin{equation}
	I_p(e_1,e_2)
	=
	\int_0^1
	\int_0^1
	d_1^\top
	\nabla_u^2 L_{t,p}
	\left(
	u_0+s d_1+r d_2
	\right)
	d_2
	\,\mathrm{d}s\,\mathrm{d}r.
	\label{eq:exact-shared-target-interaction}
\end{equation}
The finite interaction is therefore determined by the mixed curvature of the task metric over the complete two-intervention surface.
This exact identity also shows why the two edges must be considered jointly: the relevant term depends on both intervention directions $d_1$ and $d_2$.

\subsection{Local Approximation}

The exact identity depends on the Hessian throughout the intervention surface.
For sufficiently small message changes, this Hessian can be approximated by its value at the clean read input.
Define
\begin{equation}
	Q_{t,p}
	=
	\nabla_u^2 L_{t,p}
	\left(
	u_{t,p}^{\mathrm{cl}}
	\right).
	\label{eq:target-hessian}
\end{equation}
The matrix $Q_{t,p}$ describes how the sensitivity of the task metric changes locally when the shared read input of component $t$ is varied.

The following proposition states that the leading local pairwise interaction is obtained by applying this curvature to the two message changes.

\begin{proposition}[Shared-target edge interactions]
	\label{prop:shared-target}
	If the intervention surface remains within a neighborhood on which the Hessian of $L_{t,p}$ is Lipschitz, then
	\begin{equation}
		I_p(e_1,e_2)
		=
		\Delta m_{e_1,p}^{\top}
		Q_{t,p}
		\Delta m_{e_2,p}
		+
		O\!\left(
		\left(
			\|\Delta m_{e_1,p}\|
			+
			\|\Delta m_{e_2,p}\|
			\right)^3
		\right).
		\label{eq:shared-target-interaction}
	\end{equation}
\end{proposition}

The proposition has a direct interpretation.
The first term measures whether the downstream computation responds differently when the two incoming messages are changed together rather than separately.
If this term is nonzero, the leading local interaction is second-order and depends on the identities of both incoming edges.

\begin{proof}
	From \Cref{eq:exact-shared-target-interaction},
	\begin{equation}
		I_p(e_1,e_2)
		=
		\int_0^1
		\int_0^1
		d_1^\top
		\nabla_u^2 L_{t,p}
		\left(
		u_0+s d_1+r d_2
		\right)
		d_2
		\,\mathrm{d}s\,\mathrm{d}r.
	\end{equation}
	Because the Hessian is Lipschitz on the intervention surface,
	\begin{equation}
		\nabla_u^2 L_{t,p}
		\left(
		u_0+s d_1+r d_2
		\right)
		=
		Q_{t,p}
		+
		O\!\left(
		s\|d_1\|
		+
		r\|d_2\|
		\right).
	\end{equation}
	Substituting this expression into the exact identity gives
	\begin{align}
		I_p(e_1,e_2)
		 & =
		\int_0^1
		\int_0^1
		d_1^\top Q_{t,p}d_2
		\,\mathrm{d}s\,\mathrm{d}r
		\nonumber \\
		 & \quad+
		O\!\left(
		\|d_1\|
		\|d_2\|
		\left(
			\|d_1\|
			+
			\|d_2\|
			\right)
		\right).
	\end{align}
	The leading integrand is constant in $s$ and $r$, so
	\begin{equation}
		\int_0^1
		\int_0^1
		d_1^\top Q_{t,p}d_2
		\,\mathrm{d}s\,\mathrm{d}r
		=
		d_1^\top Q_{t,p}d_2.
	\end{equation}
	Furthermore,
	\begin{equation}
		\|d_1\|
		\|d_2\|
		\left(
		\|d_1\|
		+
		\|d_2\|
		\right)
		=
		O\!\left(
		\left(
		\|d_1\|
		+
		\|d_2\|
		\right)^3
		\right).
	\end{equation}
	Restoring the original notation proves \Cref{eq:shared-target-interaction}.
\end{proof}

If the mixed second-order term in \Cref{eq:shared-target-interaction} is nonzero, it is the leading local interaction.
If it vanishes, the leading interaction may occur at third or higher order.
For larger message changes, the exact identity in \Cref{eq:exact-shared-target-interaction} remains valid, but the clean-point Hessian may no longer provide an accurate approximation.

The smoothness assumption makes \Cref{prop:shared-target} a conditional local illustration rather than a general statement.
It does not apply directly at the kinks of ReLU or $L_1$, or to argmax and other discontinuous task metrics.
The finite-difference definition in \Cref{eq:position-localised-interaction} needs no such assumption and remains meaningful in those cases.

\subsection{Implications for Circuit Localization and Graph Learning}

\Cref{prop:shared-target} shows that the first potentially nonzero local interaction between two edges entering the same read input is second-order.
The matrix $Q_{t,p}$ describes how the downstream computation responds nonlinearly to changes in the shared read input.
It depends on the shared read input, token position, and downstream computation, but not on the particular pair of incoming edges.
The pair-specific information is instead contained in the message changes $\Delta m_{e_1,p}$ and $\Delta m_{e_2,p}$.
This structure is well suited to parameter sharing: a GNN can apply a common update rule across the graph while adapting its output to the features, shared components, and neighborhoods of individual edges.
The learner can therefore reuse interaction patterns across edge pairs, computation graphs, and model--task pairs.

Vanilla EAP estimates edge importance using a first-order Taylor approximation around the clean computation and therefore does not include this explicit mixed term \citep{10.18653/v1/2024.blackboxnlp-1.25}.
EAP-IG can incorporate curvature and higher-order effects by averaging gradients along an interpolation path over input embeddings, component activations, or edge masks, depending on the variant, but allocates these effects among scalar edge scores rather than identifying interactions between particular pairs \citep{10.48550/arXiv.2403.19647,10.48550/arXiv.2403.17806}.
The Integrated Hessians method instead attributes pairwise feature interactions explicitly through mixed derivatives \citep{10.48550/arXiv.2002.04138}.

This distinction matters when an edge interacts with one pathway but not another, or when multiple pathways are redundant.
Redundant OR-like mechanisms may be only partially recovered by standard circuit-discovery procedures \citep{10.48550/arXiv.2505.10039}.
Related dependencies appear in backup behavior and self-repair \citep{10.48550/arXiv.2211.00593,10.48550/arXiv.2307.15771,10.48550/arXiv.2402.15390}, as well as in copy suppression, whose effect depends on information written by earlier components \citep{10.48550/arXiv.2310.04625}.
These phenomena are not direct instances of \Cref{prop:shared-target}, but they demonstrate more generally that pathway relevance can depend on the surrounding computation.

A graph learner provides a direct mechanism for modeling such dependencies by exchanging representations between computationally related edges.
The directed line graph and original computation graph expose these relationships differently.
Composable edges are directly adjacent in the directed line graph.
However, two edges entering the same target do not compose and are therefore not directly adjacent.
If their shared target $t$ has an outgoing edge $e_3=(t,u)$, the directed line graph contains
\begin{equation}
	\bar e_1\rightarrow \bar e_3,
	\qquad
	\bar e_2\rightarrow \bar e_3.
\end{equation}
If $t$ is an internal read node, $e_3$ is an outgoing edge of the corresponding write node instead.
Under the composability relation in \Cref{app:graph-transformations}, the directed line graph then contains the same two edges.
With bidirectional message passing, information can travel between the incoming edges through $\bar e_3$ in two message-passing steps.%
\footnote{DirGNN aggregates along both directions and can realize this exchange.
	We use DAGformer only with the original computation graph.}
When GCL uses the original computation graph, both incoming edges end at component $t$.
The readout for either edge uses the learned representation of $t$ together with the ordered representation of its source and that edge's features.
The representation of $t$ can combine messages from both sources, allowing either prediction to depend on their shared context.

Once both edge representations lie within the same receptive field, the graph learner can condition the prediction for one edge directly on the observed features of the other.
Thus, the graph structure provides an inductive bias by restricting this exchange to computationally related pathways rather than treating every possible pair of transformer edges as equally relevant.

\section{Additional Experimental Details, Results, and Ablations}
\label{app:additional-results}

Our implementation is openly available at \url{https://github.com/graph-learning-circuits/graph-learning-circuits}.

\subsection{Additional Experimental Details}
\label{app:experimental-details}

We evaluate every method separately on each of the 16 held-out model--task pairs and select hyperparameters and training durations on the 25-pair cross-validation pool (\Cref{app:benchmark-split}).

\paragraph{From scores to circuits.}
Following the MIB evaluation procedure \cite{10.48550/arXiv.2504.13151}, we evaluate circuits at ten sparsity levels.
A sparsity level is the fraction of the model's edges that a circuit may retain, counted before the head promotion described below.
The ten levels range from $0.001$ to $1$, the full graph.
For node-wise Subnetwork Probing, a sparsity level is instead a fraction of node gates.
GCL, PGExplainer, EAP, EAP-IG, and both Subnetwork Probing variants produce edge or node scores: logits for GCL, PGExplainer, and Subnetwork Probing, and attribution scores for EAP and EAP-IG. 
We convert these scores into circuits by top-$k$ or backward greedy sampling \cite{10.48550/arXiv.2504.13151}, choosing based on validation data between the two procedures and between raw and absolute scores for each method. ACDC prunes edges directly and needs no sampling step. 

\paragraph{Query, key, and value granularity.}
GCL scores the query, key, and value inputs of each attention head separately, because they are distinct computational pathways, and trains with the supervision rule of \Cref{app:training}.
\textsc{InterpBench} instead evaluates these pathways at the level of the parent head. 
For held-out evaluation, we therefore apply \textsc{InterpBench}'s head-promotion rule to the selected circuits, retaining a head-level connection if any of its query, key, or value inputs is selected, so that results are comparable with the \textsc{InterpBench} baselines.

\paragraph{Metrics.}
Edge AUROC is our primary metric. 
We compute it with the trapezoidal rule over the ROC points at the ten sparsity levels, as in the MIB evaluation procedure \cite{10.48550/arXiv.2504.13151}.
\Cref{app:threshold-auroc} reports an alternative over all score thresholds.
Average precision also compares the selected circuits with the ground-truth circuit.
The integrated circuit performance ratio (CPR) and circuit--model distance (CMD) instead measure how well circuits preserve the model's behavior as edges are retained, so they need no ground truth.
CPR is the area under the faithfulness curve (higher is better), and CMD is the area between one and that curve (lower is better).
We compute both with the trapezoidal rule over the ten sparsity levels.

\paragraph{Aggregation.}
For methods run with 5 random seeds, we average the seeds within each held-out pair and then summarize across the 16 pairs. Each $\pm$ is the sample standard deviation, across seeds, of the summary statistic computed separately for each seed.
We reran all five \textsc{InterpBench} baselines on our evaluation split.

\paragraph{Selecting GCL}
For each GCL configuration, we run grouped 5-fold cross-validation over a grid of 18 hyperparameter settings, training on four folds with one seed and recording validation AUROC on the fifth at regular intervals.
At each training step where all five folds have a measurement, a setting's score is the average of its fold scores, each fold weighted by its number of groups of related pairs (\Cref{app:benchmark-split}). 
Among all 46 configurations (Tables \ref{tab:gcl-computation-configurations} and~\ref{tab:gcl-line-configurations}), this score selects the configuration, hyperparameter setting, training duration, sampling procedure, and raw or absolute logits, breaking ties in favor of the earlier step within a hyperparameter setting and the earlier-listed setting across settings.
For the ``None'' control, cross-validation selects hyperparameters and training duration separately.
We then reinitialize each selected predictor and train it on all 25 pairs with 5 seeds.

\paragraph{Selecting PGExplainer.}
Our PGExplainer adaptation fits a newly initialized MLP for every pair and seed, using only that pair's prompts and model outputs. 
To choose among 18 hyperparameter settings, we fit each setting on every cross-validation pair and score it by validation AUROC, averaged with the same group weights as for GCL.
We use the ground-truth masks of these pairs to score the settings, but not to fit the explainers.
With the selected setting, we fit the adaptation on each held-out pair with 5 seeds and use its ground-truth mask only for evaluation.

\subsection{Additional Experimental Results and Ablations}
\label{app:additional-experimental-results}

\paragraph{Edge AUROC}

\Cref{tab:additional-results} reports held-out edge AUROC for the selected GCL configuration, its control, our PGExplainer adaptation, and the baselines.
\Cref{tab:threshold-auroc} repeats this comparison with ROC curves over all score thresholds.
Tables~\ref{tab:gcl-computation-configurations} and~\ref{tab:gcl-line-configurations} report all 46 GCL configurations.

	\begin{table*}[htbp]
		\centering
		\footnotesize
		\setlength{\tabcolsep}{3.1pt}
		\renewcommand{\arraystretch}{1.04}
		\begin{adjustbox}{max width=\textwidth}
			\begin{tabular}{@{}llrr@{}}
				\toprule
				 & Method                        & Mean                             & Median                           \\
				\midrule
				\multirow{3}{*}{Our Methods}
				 & GCL (Observed)                & $0.903 \pm 0.0159$               & $0.946 \pm 0.0138$               \\
				 & GCL (None)                    & $0.877 \pm 0.0128$               & $0.926 \pm 0.0141$               \\
				 & PGExplainer                   & \secondvalue{$0.920 \pm 0.0103$} & \secondvalue{$0.950 \pm 0.0086$} \\
				\midrule
				 & ACDC                          & \secondvalue{$0.920 \pm 0.0134$} & $0.945 \pm 0.0103$               \\
				 & Subnetwork Probing, edge-wise & $0.915 \pm 0.0024$               & $0.935 \pm 0.0038$               \\
				 & Subnetwork Probing, node-wise & $0.872 \pm 0.0021$               & $0.893 \pm 0.0077$               \\
				 & EAP-IG$^{*}$                  & \bestvalue{$0.950$}              & \bestvalue{$0.973$}              \\
				 & EAP$^{*}$                     & $0.910$                          & $0.941$                          \\
				\bottomrule
			\end{tabular}
		\end{adjustbox}
		\\$^{*}$EAP and EAP-IG show negligible variability across seeds, so we report no seed standard deviation for them.
		\caption{Held-out edge AUROC across the 16 test pairs, with ROC curves built from the ten sparsity levels; aggregation as in \Cref{app:experimental-details}.
         In each column, boldface indicates the best value and underline indicates the second best value.
			\textbf{Takeaway:} Our PGExplainer adaptation and GCL rank second and third, behind EAP-IG, with regard to median edge AUROC.
		}
		\label{tab:additional-results}
	\end{table*}

	\begin{table*}[htbp]
		\centering
		\footnotesize
		\setlength{\tabcolsep}{4pt}
		\renewcommand{\arraystretch}{1.0}
		\begin{adjustbox}{max width=\textwidth}
			\begin{tabular}{@{}lllrrr@{}}
				\toprule
				              &              &          & CV                  & \multicolumn{2}{c}{Held-out test}                                  \\
				\cmidrule(lr){4-4}\cmidrule(l){5-6}
				Graph Learner & Connections  & Q/K/V    & AUROC               & Mean                              & Median                         \\
				\midrule
				\multirow{6}{*}{DirGraphConv}
				              & Observed     & Pooled   & \bestvalue{$0.919$} & $0.912 \pm 0.0086$                & \bestvalue{$0.944 \pm 0.0118$} \\
				              & Observed     & Separate & $0.893$             & \bestvalue{$0.924 \pm 0.0159$}    & $0.931 \pm 0.0166$             \\
				              & None         & Pooled   & $0.896$             & $0.885 \pm 0.0273$                & $0.905 \pm 0.0533$             \\
				              & None         & Separate & $0.849$             & $0.899 \pm 0.0093$                & $0.917 \pm 0.0135$             \\
				              & All possible & Pooled   & $0.900$             & $0.912 \pm 0.0145$                & $0.940 \pm 0.0179$             \\
				              & All possible & Separate & $0.888$             & $0.909 \pm 0.0147$                & $0.931 \pm 0.0100$             \\
				\midrule
				\multirow{6}{*}{DirGCN}
				              & Observed     & Pooled   & \bestvalue{$0.913$} & $0.927 \pm 0.0126$                & \bestvalue{$0.962 \pm 0.0109$} \\
				              & Observed     & Separate & $0.910$             & \bestvalue{$0.929 \pm 0.0137$}    & $0.954 \pm 0.0080$             \\
				              & None         & Pooled   & $0.883$             & $0.896 \pm 0.0218$                & $0.936 \pm 0.0092$             \\
				              & None         & Separate & $0.870$             & $0.883 \pm 0.0370$                & $0.907 \pm 0.0138$             \\
				              & All possible & Pooled   & $0.902$             & $0.908 \pm 0.0206$                & $0.937 \pm 0.0173$             \\
				              & All possible & Separate & $0.868$             & $0.887 \pm 0.0183$                & $0.908 \pm 0.0194$             \\
				\midrule
				\multirow{6}{*}{DirSAGE}
				              & Observed     & Pooled   & \bestvalue{$0.909$} & \bestvalue{$0.917 \pm 0.0062$}    & \bestvalue{$0.942 \pm 0.0081$} \\
				              & Observed     & Separate & $0.899$             & $0.906 \pm 0.0100$                & $0.929 \pm 0.0038$             \\
				              & None         & Pooled   & $0.888$             & $0.900 \pm 0.0220$                & $0.925 \pm 0.0218$             \\
				              & None         & Separate & $0.877$             & $0.887 \pm 0.0287$                & $0.914 \pm 0.0058$             \\
				              & All possible & Pooled   & $0.892$             & $0.893 \pm 0.0217$                & $0.937 \pm 0.0119$             \\
				              & All possible & Separate & $0.868$             & $0.868 \pm 0.0139$                & $0.896 \pm 0.0286$             \\
				\midrule
				\multirow{6}{*}{DirGAT}
				              & Observed     & Pooled   & $0.903$             & \bestvalue{$0.929 \pm 0.0107$}    & $0.953 \pm 0.0144$             \\
				              & Observed     & Separate & $0.899$             & $0.917 \pm 0.0118$                & $0.944 \pm 0.0113$             \\
				              & None         & Pooled   & $0.900$             & $0.914 \pm 0.0114$                & $0.947 \pm 0.0185$             \\
				              & None         & Separate & $0.884$             & $0.901 \pm 0.0159$                & $0.920 \pm 0.0114$             \\
				              & All possible & Pooled   & \bestvalue{$0.923$} & $0.917 \pm 0.0029$                & \bestvalue{$0.955 \pm 0.0067$} \\
				              & All possible & Separate & $0.907$             & $0.906 \pm 0.0153$                & $0.939 \pm 0.0133$             \\
				\midrule
				\multirow{6}{*}{DAGformer}
				              & Observed     & Pooled   & $0.865$             & $0.848 \pm 0.0241$                & $0.862 \pm 0.0592$             \\
				              & Observed     & Separate & $0.867$             & $0.814 \pm 0.0445$                & $0.842 \pm 0.0434$             \\
				              & None         & Pooled   & $0.899$             & $0.847 \pm 0.0216$                & $0.889 \pm 0.0191$             \\
				              & None         & Separate & $0.888$             & $0.841 \pm 0.0300$                & $0.904 \pm 0.0228$             \\
				              & All possible & Pooled   & $0.896$             & \bestvalue{$0.890 \pm 0.0306$}    & \bestvalue{$0.935 \pm 0.0264$} \\
				              & All possible & Separate & \bestvalue{$0.931$} & $0.796 \pm 0.0402$                & $0.816 \pm 0.0814$             \\
				\bottomrule
			\end{tabular}
		\end{adjustbox}
		\caption{Cross-validation and held-out results for GCL on the \textbf{original computation graph}.
            Cross-validation (CV) AUROC is the selection score and held-out statistics are aggregated as in Section ~\ref{app:experimental-details}; connection settings are defined in Section ~\ref{app:graph-learning}.
            In each column, boldface indicates the best value within each graph learner.
			\textbf{Takeaway:} Observed connections beat no connections in CV AUROC for every DirGNN learner, but not for DAGformer.
		}
		\label{tab:gcl-computation-configurations}
	\end{table*}

	\begin{table*}[htbp]
		\centering
		\footnotesize
		\setlength{\tabcolsep}{4pt}
		\renewcommand{\arraystretch}{1.0}
		\begin{adjustbox}{max width=\textwidth}
			\begin{tabular}{@{}lllrrr@{}}
				\toprule
				              &             &          & CV                  & \multicolumn{2}{c}{Held-out test}                                  \\
				\cmidrule(lr){4-4}\cmidrule(l){5-6}
				Graph Learner & Connections & Q/K/V    & AUROC               & Mean                              & Median                         \\
				\midrule
				\multirow{4}{*}{DirGraphConv}
				              & Observed    & Pooled   & \bestvalue{$0.908$} & $0.888 \pm 0.0147$                & \bestvalue{$0.921 \pm 0.0185$} \\
				              & Observed    & Separate & \bestvalue{$0.908$} & $0.888 \pm 0.0145$                & $0.906 \pm 0.0246$             \\
				              & None        & Pooled   & $0.876$             & \bestvalue{$0.893 \pm 0.0089$}    & \bestvalue{$0.921 \pm 0.0136$} \\
				              & None        & Separate & $0.853$             & $0.866 \pm 0.0148$                & $0.878 \pm 0.0260$             \\
				\midrule
				\multirow{4}{*}{DirGCN}
				              & Observed    & Pooled   & \bestvalue{$0.942$} & \bestvalue{$0.903 \pm 0.0159$}    & \bestvalue{$0.946 \pm 0.0138$} \\
				              & Observed    & Separate & $0.918$             & $0.872 \pm 0.0231$                & $0.893 \pm 0.0430$             \\
				              & None        & Pooled   & $0.881$             & $0.877 \pm 0.0128$                & $0.926 \pm 0.0141$             \\
				              & None        & Separate & $0.851$             & $0.871 \pm 0.0181$                & $0.877 \pm 0.0086$             \\
				\midrule
				\multirow{4}{*}{DirSAGE}
				              & Observed    & Pooled   & $0.922$             & $0.891 \pm 0.0096$                & \bestvalue{$0.948 \pm 0.0086$} \\
				              & Observed    & Separate & \bestvalue{$0.927$} & \bestvalue{$0.899 \pm 0.0159$}    & $0.915 \pm 0.0174$             \\
				              & None        & Pooled   & $0.881$             & $0.879 \pm 0.0158$                & $0.928 \pm 0.0094$             \\
				              & None        & Separate & $0.838$             & $0.861 \pm 0.0239$                & $0.881 \pm 0.0200$             \\
				\midrule
				\multirow{4}{*}{DirGAT}
				              & Observed    & Pooled   & \bestvalue{$0.934$} & $0.865 \pm 0.0327$                & $0.903 \pm 0.0536$             \\
				              & Observed    & Separate & \bestvalue{$0.934$} & \bestvalue{$0.898 \pm 0.0106$}    & \bestvalue{$0.932 \pm 0.0119$} \\
				              & None        & Pooled   & $0.892$             & $0.872 \pm 0.0141$                & $0.918 \pm 0.0216$             \\
				              & None        & Separate & $0.865$             & $0.860 \pm 0.0295$                & $0.877 \pm 0.0291$             \\
				\bottomrule
			\end{tabular}
		\end{adjustbox}
		\caption{Cross-validation and held-out results for GCL with the \textbf{directed line graph}.
            The procedure is the same as that used for Table \ref{tab:gcl-computation-configurations}.
			\textbf{Takeaway:} Observed connections beat no connections in CV AUROC for every graph learner, and cross-validation selects DirGCN with pooled supervision.
		}
		\label{tab:gcl-line-configurations}
	\end{table*}

\paragraph{Average precision and behavior-based metrics.}
\Cref{tab:secondary-metrics} reports average precision, CMD, and CPR. 
GCL is trained to recover the known circuit edges, whereas Subnetwork Probing learns sparse masks that preserve task behavior.
GCL performs comparably to edge-wise and better than node-wise Subnetwork Probing on AUROC and average precision, while both Subnetwork Probing variants score better on CMD and CPR. 
Because a subgraph can preserve behavior without containing the known circuit's edges, the two kinds of metrics complement each other, and combining ground-truth supervision with a faithfulness objective may improve GCL on both.
\\ 

	\begin{table*}[htbp]
		\centering
		\footnotesize
		\setlength{\tabcolsep}{4pt}
		\renewcommand{\arraystretch}{1.04}
		\begin{adjustbox}{max width=\textwidth}
			\begin{tabular}{@{}llrrr@{}}
				\toprule
				 & Method                        & Average precision $\uparrow$     & CMD $\downarrow$                 & CPR $\uparrow$                   \\
				\midrule
				\multirow{3}{*}{Our Methods}
				 & GCL (Observed)                & $0.489 \pm 0.0406$               & $0.179 \pm 0.0243$               & $0.820 \pm 0.0243$               \\
				 & GCL (None)                    & $0.379 \pm 0.0125$               & $0.190 \pm 0.0050$               & $0.809 \pm 0.0050$               \\
				 & PGExplainer                   & $0.553 \pm 0.0124$               & $0.105 \pm 0.0075$               & $0.894 \pm 0.0075$               \\
				\midrule
				 & ACDC                          & \secondvalue{$0.581 \pm 0.0123$} & \bestvalue{$0.043 \pm 0.0003$}   & \bestvalue{$0.956 \pm 0.0003$}   \\
				 & Subnetwork Probing, edge-wise & $0.509 \pm 0.0116$               & \secondvalue{$0.047 \pm 0.0014$} & \secondvalue{$0.952 \pm 0.0014$} \\
				 & Subnetwork Probing, node-wise & $0.352 \pm 0.0073$               & $0.138 \pm 0.0032$               & $0.861 \pm 0.0032$               \\
				 & EAP-IG$^{*}$                  & \bestvalue{$0.615$}              & $0.054$                          & $0.945$                          \\
				 & EAP$^{*}$                     & $0.487$                          & $0.168$                          & $0.831$                          \\
				\bottomrule
			\end{tabular}
		\end{adjustbox}
		{\footnotesize $^{*}$EAP and EAP-IG show negligible variability across seeds, so we report no seed standard deviation for them.\par}
		\caption{ Average precision and behavior-based metrics, as means across the 16 held-out pairs; aggregation as in Section ~\ref{app:experimental-details}. 
        In each column, boldface indicates the best value and underline indicates the second best value.
			\textbf{Takeaway:} Our PGExplainer adaptation achieves lower CMD and higher CPR than the GCL methods.
		}
		\label{tab:secondary-metrics}
	\end{table*}

\paragraph{Precision and recall.}
\Cref{tab:precision-recall} reports held-out precision and recall at ten sparsity levels; and when each method uses the same selection procedure but retains no more edges than the ground-truth circuit. Edges are counted after head promotion in both circuits. These results show that a high edge AUROC does not guarantee that a circuit no larger than the ground-truth circuit recovers most ground-truth edges.
As in link prediction \citep{10.1073/pnas.2312527121}, a high AUROC can coexist with low precision and recall among the highest-ranked predictions when the ground truth is sparse, and the median ground-truth circuit contains 5.7\% of the edges.
Under this restriction, GCL reaches a median precision of $0.534$, below ACDC, EAP-IG, and EAP.
Its median recall of $0.462$ is below every method except node-wise Subnetwork Probing and the ``None'' control.
ACDC and EAP exceed GCL on both measures even though their median edge AUROC is slightly lower.
Removing message passing lowers precision to $0.381$ and  recall to $0.300$, larger drops than in median edge AUROC ($0.946$ to $0.926$).

	\begin{table*}[htbp]
		\centering
		\footnotesize
		\setlength{\tabcolsep}{3pt}
		\begin{adjustbox}{max width=\textwidth}
			\begin{tabular}{@{}llrrrrrrrrrrr@{}}
				\toprule
				 &                               & \multicolumn{10}{c}{Sparsity level} &                                                                                                                                                                                                                                         \\
				\cmidrule(lr){3-12}
				 & Method                        & 0.001                               & 0.002               & 0.005               & 0.01                  & 0.02                  & 0.05                  & 0.1                   & 0.2                   & 0.5                   & 1                   & $\leq$ GT size       \\
				\midrule
				\multicolumn{13}{@{}l}{\emph{\textbf{Precision}}}                                                                                                                                                                                                                                                                         \\
				\multirow{3}{*}{Our Methods}
				 & GCL (Observed)                & $0.000$                             & $0.000$             & $0.000$             & \secondvalue{$0.290$} & \secondvalue{$0.472$} & $0.374$               & $0.301$               & \secondvalue{$0.185$} & $0.091$               & $0.057$             & $0.534$               \\
				 & GCL (None)                    & $0.000$                             & $0.000$             & $0.000$             & $0.100$               & $0.270$               & $0.339$               & $0.276$               & \secondvalue{$0.185$} & \secondvalue{$0.101$} & $0.057$             & $0.381$               \\
				 & PGExplainer                   & $0.000$                             & $0.000$             & $0.000$             & $0.184$               & $0.410$               & \secondvalue{$0.453$} & $0.305$               & $0.171$               & $0.084$               & $0.057$             & $0.500$               \\
				\cmidrule(l){1-13}
				 & ACDC                          & $0.000$                             & $0.000$             & $0.000$             & $0.050$               & $0.310$               & $0.420$               & \bestvalue{$0.343$}   & \bestvalue{$0.220$}   & \bestvalue{$0.154$}   & $0.058$             & \bestvalue{$0.708$}   \\
				 & Subnetwork Probing, edge-wise & $0.000$                             & $0.000$             & $0.000$             & $0.000$               & $0.289$               & $0.391$               & $0.236$               & $0.129$               & $0.068$               & $0.057$             & $0.500$               \\
				 & Subnetwork Probing, node-wise & \bestvalue{$0.412$}                 & \bestvalue{$0.412$} & \bestvalue{$0.412$} & \bestvalue{$0.412$}   & $0.358$               & $0.371$               & \secondvalue{$0.335$} & $0.178$               & $0.073$               & $0.057$             & $0.500$               \\
				 & EAP-IG                        & $0.000$                             & $0.000$             & $0.000$             & $0.000$               & $0.444$               & \bestvalue{$0.477$}   & $0.289$               & $0.142$               & $0.070$               & $0.057$             & $0.583$               \\
				 & EAP                           & $0.000$                             & $0.000$             & $0.000$             & $0.000$               & \bestvalue{$0.633$}   & $0.387$               & $0.228$               & $0.130$               & $0.071$               & $0.057$             & \secondvalue{$0.600$} \\
				\midrule
                \midrule
				\multicolumn{13}{@{}l}{\emph{\textbf{Recall}}}                                                                                                                                                                                                                                                                            \\
				\multirow{3}{*}{Our Methods}
				 & GCL (Observed)                & $0.000$                             & $0.000$             & $0.000$             & \secondvalue{$0.107$} & $0.390$               & $0.688$               & $0.892$               & \secondvalue{$0.975$} & $1.000$               & $1.000$             & $0.462$               \\
				 & GCL (None)                    & $0.000$                             & $0.000$             & $0.000$             & $0.027$               & $0.189$               & $0.525$               & $0.868$               & $0.950$               & $1.000$               & $1.000$             & $0.300$               \\
				 & PGExplainer                   & $0.000$                             & $0.000$             & $0.000$             & $0.062$               & \secondvalue{$0.475$} & \secondvalue{$0.788$} & $0.875$               & \bestvalue{$1.000$}   & $1.000$               & $1.000$             & $0.487$               \\
				\cmidrule(l){1-13}
				 & ACDC                          & $0.000$                             & $0.000$             & $0.000$             & $0.010$               & $0.096$               & $0.762$               & \secondvalue{$0.978$} & \bestvalue{$1.000$}   & $1.000$               & $1.000$             & $0.562$               \\
				 & Subnetwork Probing, edge-wise & $0.000$                             & $0.000$             & $0.000$             & $0.000$               & $0.458$               & $0.745$               & $0.950$               & \bestvalue{$1.000$}   & $1.000$               & $1.000$             & $0.473$               \\
				 & Subnetwork Probing, node-wise & \bestvalue{$0.350$}                 & \bestvalue{$0.350$} & \bestvalue{$0.350$} & \bestvalue{$0.350$}   & $0.388$               & $0.522$               & $0.786$               & \bestvalue{$1.000$}   & $1.000$               & $1.000$             & $0.362$               \\
				 & EAP-IG                        & $0.000$                             & $0.000$             & $0.000$             & $0.000$               & \bestvalue{$0.500$}   & \bestvalue{$0.929$}   & \bestvalue{$1.000$}   & \bestvalue{$1.000$}   & $1.000$               & $1.000$             & \secondvalue{$0.583$} \\
				 & EAP                           & $0.000$                             & $0.000$             & $0.000$             & $0.000$               & $0.450$               & $0.718$               & $0.854$               & \bestvalue{$1.000$}   & $1.000$               & $1.000$             & \bestvalue{$0.586$}   \\
				\bottomrule
			\end{tabular}
		\end{adjustbox}
		\caption{Held-out precision and recall at ten sparsity levels and with circuits no larger than the ground-truth circuit, as medians across the 16 pairs; aggregation as in Section ~\ref{app:experimental-details}.
            The $\leq$ GT size column restricts each circuit to at most as many edges as the ground-truth (GT) circuit, both counted after head promotion.
            In each column of each panel, boldface indicates the best value and underline indicates the second best value. 
            The precision of an empty circuit is 0.
            Up to sparsity 0.005, median precision and recall are 0 for every method except node-wise Subnetwork Probing.
            At sparsity level 1, every method except ACDC returns the full graph, whereas ACDC returns the circuit it selects on validation data, which may be smaller; we therefore mark no best value in that precision column.
			\textbf{Takeaway:} With circuits no larger than the ground-truth circuit, ACDC, EAP, and EAP-IG recover more ground-truth edges than GCL despite similar edge AUROC.}
		\label{tab:precision-recall}
	\end{table*}

\subsection[Edge AUROC over All Score Thresholds]{Edge AUROC over All Score Thresholds}
\label{app:threshold-auroc}

The main evaluation builds every ROC curve from circuits at the same ten sparsity levels, following MIB \citep{10.48550/arXiv.2504.13151}.
Here, we instead threshold GCL, PGExplainer, EAP-IG, and EAP at every distinct edge score.
ACDC and Subnetwork Probing provide one ROC point for each of the 16 values of their sweep.
\Cref{tab:threshold-auroc} reports the results. 

	\begin{table*}[htbp]
		\centering
		\footnotesize
		\setlength{\tabcolsep}{3.4pt}
		\renewcommand{\arraystretch}{1.04}
		\begin{adjustbox}{max width=\textwidth}
			\begin{tabular}{@{}llrr@{}}
				\toprule
				 & Method                        & Mean                           & Median                         \\
				\midrule
				\multirow{3}{*}{Our Methods}
				 & GCL (Observed)                & $0.914 \pm 0.0117$             & $0.945 \pm 0.0173$             \\
				 & GCL (None)                    & $0.866 \pm 0.0086$             & $0.858 \pm 0.0167$             \\
				 & PGExplainer                   & $0.879 \pm 0.0058$             & $0.887 \pm 0.0061$             \\
				\midrule
				 & ACDC                          & \bestvalue{$0.956 \pm 0.0019$} & \bestvalue{$0.979 \pm 0.0011$} \\
				 & Subnetwork Probing, edge-wise & $0.847 \pm 0.0140$             & $0.855 \pm 0.0204$             \\
				 & Subnetwork Probing, node-wise & $0.819 \pm 0.0070$             & $0.848 \pm 0.0197$             \\
				 & EAP-IG$^{*}$                  & \secondvalue{$0.949$}          & \secondvalue{$0.975$}          \\
				 & EAP$^{*}$                     & $0.896$                        & $0.960$                        \\
				\bottomrule
			\end{tabular}
		\end{adjustbox}
		\par\smallskip
		{\footnotesize $^{*}$EAP and EAP-IG show negligible variability across seeds, so we report no seed standard deviation for them.\par}
		\caption{As Table~\ref{tab:additional-results}, with ROC curves over all score thresholds.
            The GCL rows use the configuration that cross-validation selects under this evaluation: DirGAT on the directed line graph with separate query, key, and value supervision.
			  In each column, boldface indicates the best value and underline indicates the second best value. 
			The ``None'' GCL row uses this configuration with all message-passing connections removed.
			\textbf{Takeaway:} Over all score thresholds, GCL and our PGExplainer adaptation have lower median edge AUROC than ACDC, EAP-IG, and EAP, and removing message passing lowers the median of GCL from $0.945$ to $0.858$. 
		}
		\label{tab:threshold-auroc}
	\end{table*}

This evaluation changes the ranking: ACDC moves to the highest median, and PGExplainer falls from just above GCL to below it. 
Cross-validation also selects a different GCL configuration, and removing its message passing costs
more, lowering the median from 0.945 to 0.858 rather than from 0.946 to 0.926.
We treat the sparsity-level evaluation as primary because it evaluates every method on circuits at the same ten sparsity levels, corresponding to the levels at which MIB computes CPR and CMD. 
Over all thresholds, the ACDC and Subnetwork Probing curves have at most 16 points from their sweep, plus the two endpoints, versus one per distinct score for the other methods. 

\section{Computational Costs}
\label{app:complexity}

We use the standard work--span model of parallel complexity \citep{10.1145/227234.227246}.
For a method $X$, $W_X$ is its total computational effort, and $S_X$ is its span in terms of the number of sequential steps along the longest chain of dependencies.

We compare the work and span required to localize a circuit for one new model--task pair.
For Graph Circuit Learning (GCL), this is its amortized inference cost after the shared predictor has been trained.
We omit the cost of model selection or predictor training.

Let $(W_{\mathrm{feat}},S_{\mathrm{feat}})$ denote the work and span of the fixed GCL feature construction for one pair, and let $(W_{\mathrm{pred}},S_{\mathrm{pred}})$ denote the work and span of feature alignment, constructing the directed line graph when used, one trained-predictor evaluation, pooling, and readout.
Let $W_{\mathrm{EAP}}$ denote the work of one EAP calculation.
$W_{\mathrm{feat}}$ and $W_{\mathrm{EAP}}$ have the same asymptotic order of target-model work for matched support sizes, because each runs a fixed number of forward and backward passes of the target model per prompt pair.
Thus, GCL has per-pair costs:
\begin{equation}
	\left(W_{\mathrm{GCL}},S_{\mathrm{GCL}}\right)
	=
	\left(W_{\mathrm{feat}}+W_{\mathrm{pred}},S_{\mathrm{feat}}+S_{\mathrm{pred}}\right).
	\label{eq:complexity-gcl-eap}
\end{equation}

EAP contracts its attribution quantities directly into edge scores \citep{10.18653/v1/2024.blackboxnlp-1.25}, whereas GCL adds the predictor cost in \Cref{eq:complexity-gcl-eap}.
Therefore, we do not claim a target-model-cost advantage over EAP.

EAP-IG evaluates gradients at $K$ interpolation points \citep{10.48550/arXiv.2403.19647}.
We use the AutoCircuit implementation of EAP-IG \citep{10.48550/arXiv.2407.08734}, as \textsc{InterpBench} does \citep{10.52202/079017-2950}.
Let $(W_{\mathrm{IG,pre}},S_{\mathrm{IG,pre}})$ and $(W_{\mathrm{IG,reduce}},S_{\mathrm{IG,reduce}})$ be its shared preprocessing and final reduction costs, respectively.
Writing $(W_k,S_k)$ for the cost at integration point $k$, its work and ideal parallel span are
\begin{equation}
	\begin{aligned}
		W_{\mathrm{IG}}
		 & =
		W_{\mathrm{IG,pre}}+\sum_{k=1}^{K}W_k+W_{\mathrm{IG,reduce}},
		\\
		S_{\mathrm{IG}}^{\mathrm{parallel}}
		 & =
		S_{\mathrm{IG,pre}}+\max_k S_k+S_{\mathrm{IG,reduce}},
		\quad
		S_{\mathrm{IG,reduce}}=\mathcal{O}(\log K).
	\end{aligned}
\end{equation}

For $X\in\{\mathrm{PG},\mathrm{SP}\}$, let $U_X$ be the number of optimizer updates used to fit our PGExplainer adaptation or Subnetwork Probing instance for one pair.
Its work and span are
\begin{equation}
	W_X=W_{X,\mathrm{init}}+\sum_{u=1}^{U_X}W_{X,u}+W_{X,\mathrm{out}},
	\qquad
	S_X=S_{X,\mathrm{init}}+\sum_{u=1}^{U_X}S_{X,u}+S_{X,\mathrm{out}}.
\end{equation}
Our PGExplainer adaptation updates MLP parameters, whereas Subnetwork Probing updates a directly parameterized mask.
Within each update, all mask variables are optimized jointly through a common objective.
In both cases, update $u+1$ depends on the parameters produced by update $u$, so the updates are sequential \citep{10.48550/arXiv.2011.04573,10.18653/v1/2021.naacl-main.74}.

At a fixed ACDC threshold $\tau$, let $N_{\tau}$ be the number of edges examined along its greedy trajectory, and let $(W_{\tau,j},S_{\tau,j})$ be the cost of the intervention for the $j$-th examined edge.
Its work and span are
\begin{equation}
	W_{\mathrm{ACDC}}(\tau)=W_{\mathrm{pre}}(\tau)+\sum_{j=1}^{N_{\tau}} W_{\tau,j},
	\qquad
	S_{\mathrm{ACDC}}(\tau)=S_{\mathrm{pre}}(\tau)+\sum_{j=1}^{N_{\tau}} S_{\tau,j}.
\end{equation}
The intervention costs can differ by edge, and each intervention is evaluated against the circuit state produced by earlier accepted removals, so later decisions depend on earlier ones \citep{10.52202/075280-0719}.

GCL performs target-model tracing with the same asymptotic order of target-model work as EAP. %
Unlike EAP-IG, it does not require repeated integration-point evaluations.
Unlike Subnetwork Probing, it does not require a per-pair optimization loop.
Unlike ACDC, it does not follow a sequential greedy intervention trajectory.
This comparison characterizes marginal algorithmic work and dependency structure rather than establishing a universal wall-clock or total-cost advantage.
Such an advantage also depends on predictor-training amortization, the scaling of $W_{\mathrm{pred}}$, hardware, batching, and model and graph sizes.

\section{Extended Related Work}
\label{app:extended-related-work}

\paragraph{Circuit localization methods.}
Automated methods for circuit localization reduce the manual effort required to select components or connections.
ACDC tests removals of individual edges with causal interventions, whereas EAP and EAP-IG estimate edge effects from gradients \citep{10.52202/075280-0719,10.18653/v1/2024.blackboxnlp-1.25,10.48550/arXiv.2403.19647,10.48550/arXiv.2403.17806}.
Subnetwork Probing and Edge Pruning instead optimize directly parameterized circuit masks \citep{10.18653/v1/2021.naacl-main.74,10.52202/079017-0587}.
These methods estimate or optimize a new circuit for each model--task pair and therefore provide the most direct baselines for Graph Circuit Learning.

Recent methods have also applied learning to circuit localization, although they differ in what is learned and whether it is reused.
MechRL trains a policy with rewards from causal interventions \citep{10.48550/arXiv.2605.26343}.
Its learned policy can be applied to unseen behaviors without further training, whereas its warm-start variant adapts the policy to the new behavior.
Differentiable Faithfulness Alignment learns to map node-importance scores from a source model to a target model, using a differentiable faithfulness objective on the target \citep{10.48550/arXiv.2604.24302}.
Universal Circuits jointly learn an aligned feature space across vision transformers and use it to identify class circuits that are shared across models or specific to individual models \citep{SliNacpKqc}.
The resulting circuits are defined over learned features rather than component connections.
CircuitLasso does not reuse a localization model across cases.
It fits a sparse regression model to activations from one model and task, producing circuits over neurons or sparse-autoencoder features \citep{10.48550/arXiv.2606.16939}.

By contrast, Graph Circuit Learning (GCL) is trained on semi-synthetic model--task pairs with known circuits and then applied to unseen pairs.
Beyond these differences in supervision and transfer, our study asks whether message passing over computation-graph structure improves circuit localization.
This question motivates the connections to GNN explainers and weight-space learning.

\paragraph{GNN explainers.}
Post-hoc GNN explainers identify which part of an input graph supports a prediction made by a fixed graph neural network.
GNNExplainer optimizes a subgraph for each queried prediction.
SubgraphX searches candidate subgraphs for a given query \citep{10.48550/arXiv.1903.03894,10.48550/arXiv.2102.05152}.
PGExplainer trains one edge-mask generator across inputs, allowing it to explain a new input without another round of mask fitting \citep{10.48550/arXiv.2011.04573}.
GraphMask likewise learns reusable gates that remove messages from a fixed GNN while preserving its predictions \citep{10.48550/arXiv.2010.00577}.
Gem derives training targets from the effects of edge removals and then trains a generator that explains new inputs \citep{10.48550/arXiv.2104.06643}.

Among these methods, PGExplainer provides a parameterized edge scorer and therefore offers a useful template for adapting graph explainability methods to circuit localization.
Our adaptation is fitted independently to each model--task pair.

\paragraph{Weight-space learning.}
Weight-space learning treats neural weights as a meaningful domain for analysis and modeling \citep{10.48550/arXiv.2603.10090}.
Permutation Equivariant Neural Functionals process model weights while respecting the fact that hidden units can be reordered without changing the represented function \citep{10.52202/075280-1085}.
Graph Metanetworks encode a neural network as a graph, allowing one metanetwork to process different architectures while respecting parameter reorderings that leave the represented function unchanged \citep{10.48550/arXiv.2312.04501}.
Such methods  describe or modify the network as a whole, for example by predicting a property or editing its parameters.
GCL adopts the same broad view of a trained network as structured data, but adds activations and gradients from a specified task.
Its output is one score for each candidate circuit connection rather than a prediction or transformation at the level of the whole network.

\paragraph{Foundation models and cross-domain transfer.}
TabPFN provides a related example of training once across many tasks and applying the resulting predictor to unseen tasks.
It is pretrained on millions of synthetic tabular datasets sampled from a prior, then uses labeled examples from an unseen dataset as context to predict its unlabeled examples without fitting a new model \citep{10.1038/s41586-024-08328-6}.
Graph foundation models pursue a related form of reuse across graph domains.
GFT pretrains a transferable vocabulary of computation-tree patterns for downstream graph tasks, while Finkelshtein et al.\ construct graph foundation models with node- and label-permutation equivariance and feature-permutation invariance and evaluate them in zero-shot node classification across graphs with different features and labels \citep{10.52202/079017-3412,10.48550/arXiv.2506.14291}.

\paragraph{Program-based and formal mechanistic interpretability.}
Restricted Access Sequence Processing (RASP)  expresses sequence algorithms through operations that transformers can implement, and \textsc{Tracr} compiles RASP programs into transformer weights with known internal structure \citep{10.48550/arXiv.2106.06981,10.52202/075280-1649}.
\textsc{TracrBench} provides a large collection of validated RASP programs and compiled models.
\textsc{InterpBench} trains models to preserve program-derived circuits while giving them weights and activations closer to those of conventionally trained models \citep{10.48550/arXiv.2409.13714,10.52202/079017-2950}.
Together, these controlled models make it possible to study whether circuit labels can be learned across cases rather than inferred anew for every case.

Other work seeks to recover a program rather than a circuit.
Transformer Programs restrict the model and its training procedure so that the trained network can be converted into a discrete program \citep{10.52202/075280-2131}.
Neural Decompiling trains a model to recover complete RASP programs from \textsc{Tracr} weights \citep{10.1007/978-3-031-71602-7_3}.
\citet{10.48550/arXiv.2602.08857} instead reparameterize a trained transformer as a RASP program and use causal interventions to isolate a small sufficient subprogram, recovering interpretable programs from small transformers trained on algorithmic and formal-language tasks.
Formal approaches pursue stronger guarantees.
\citet{10.52202/079017-2523} use a mechanistic interpretation of a model to prove lower bounds on its accuracy.
\citet{10.48550/arXiv.2602.16823} use neural network verification to discover circuits with provable guarantees: agreement with the model over a continuous input region, robustness of that agreement under patching perturbations, and minimality.
Both lines of work currently operate on small models, because verification and proof techniques are challenging to scale to modern transformers.
These approaches certify behavior of either the whole model or a discovered circuit rather than supplying a unique ground-truth circuit.
To the best of our knowledge, benchmarks whose circuits are known by construction provide the clearest established source of ground-truth supervision for circuit localization across model--task pairs.

\Needspace{26\baselineskip}
\section{Generative AI Usage}

Table~\ref{tab:aiassists} summarizes our use of generative AI assistants.
We take full responsibility for all content in this work.
We independently verified all output from the AI assistants and made all final methodological and interpretive decisions.

\begin{table}[!h]
    \footnotesize
	\centering
	\begin{tabular}{@{}ll@{}}
		\toprule
		AI Assistant Provider & Usage                   \\
		\midrule
		Anthropic                      & Theoretical analysis                 \\
		OpenAI                         & Computational complexity analysis    \\
		Anthropic \& OpenAI            & Generation of experiment code        \\
		Anthropic \& OpenAI            & Generation of figures and tables     \\
		Anthropic \& OpenAI            & Proofreading of the final manuscript \\
		\bottomrule
	\end{tabular}
	\caption{
		From Anthropic, we used Claude Opus 4.8, Claude Opus 5, Claude Opus 5.5, Claude Fable 5, and Claude Fable 5.1.
		From OpenAI, we used GPT-5.5, GPT-5.6 Sol, GPT-5.6 Terra, GPT-6 Astra, and GPT-6.1 Sol.
	}
	\label{tab:aiassists}
\end{table}

\end{document}